\documentclass[twoside,11pt]{article}
\usepackage{jair}
\usepackage{theapa}

\usepackage{graphics,latexsym,amsmath,amssymb,theorem,enumerate,array,multirow,float,graphicx,graphics,wasysym,stmaryrd,bbm}
\usepackage{appendix}

\def\imagetop#1{\vtop{\null\hbox{#1}}}

\firstpageno{1}

\def\defemb#1#2{\expandafter\def\csname #1\endcsname
                              {\relax\ifmmode #2\else\hbox{$#2$}\fi}}
\defemb{cA}{{\cal A}}
\defemb{cB}{{\cal B}}
\defemb{cC}{{\cal C}}
\defemb{cc}{{\cal c}}
\defemb{cD}{{\cal D}}
\defemb{cE}{{\cal E}}
\defemb{cF}{{\cal F}}
\defemb{cf}{{\cal f}}
\defemb{cv}{{\cal v}}
\defemb{cx}{{\cal x}}

\defemb{cG}{{\cal G}}
\defemb{GG}{{\cal G}}

\defemb{cI}{{\cal I}}
\defemb{cK}{{\cal K}}
\defemb{cP}{{\cal P}}
\defemb{cL}{{\cal L}}
\defemb{cN}{{\cal N}}
\defemb{cM}{{\cal M}}
\defemb{cR}{{\cal R}}

\defemb{cS}{{\cal S}}
\defemb{SS}{{\cal S}}

\defemb{cT}{{\cal T}}
\defemb{cV}{{\cal V}}
\defemb{cU}{{\cal U}}
\defemb{cO}{{\cal O}}
\defemb{cH}{{\cal H}}
\defemb{cX}{{\cal X}}
\defemb{cY}{{\cal Y}}
\defemb{cZ}{{\cal Z}}

\defemb{sP}{{\sf P}}
\defemb{sG}{{\sf G}}
\defemb{sU}{{\sf U}}
\defemb{sE}{{\sf E}}
\defemb{sI}{{\sf I}}
\defemb{su}{{\sf u}}
\defemb{se}{{\sf e}}
\defemb{ss}{{\sf s}}
\defemb{st}{{\sf t}}
\defemb{sp}{{\sf p}}
\defemb{si}{{\sf i}}
\defemb{sb}{{\sf b}}
\defemb{sw}{{\sf w}}
\defemb{sA}{{\sf A}}
\defemb{sB}{{\sf B}}
\defemb{sa}{{\sf a}}
\defemb{sb}{{\sf b}}
\defemb{scc}{{\sf c}}
\defemb{sd}{{\sf d}}

\defemb{tta}{{\tt a}}
\defemb{ttb}{{\tt b}}
\defemb{ttc}{{\tt c}}
\defemb{ttd}{{\tt d}}
\defemb{tte}{{\tt e}}
\defemb{ttf}{{\tt f}}

\defemb{bara}{{\bar{a}}}
\defemb{barb}{{\bar{b}}}
\defemb{barc}{{\bar{c}}}
\defemb{barf}{{\bar{f}}}
\defemb{barv}{{\bar{v}}}

\newcommand{\argmax}{\operatornamewithlimits{argmax}}

 \newenvironment{proof}{\noindent {\bf Proof:}}%
 {\hfill \rule[0.3ex]{1ex}{1ex} \par \addvspace{\bigskipamount}}

 \newenvironment{silentproof}{\noindent}%
 {\hfill \rule[0.3ex]{1ex}{1ex} \par \addvspace{\bigskipamount}}

\newtheorem{theorem}{Theorem}
\newtheorem{lemma}[theorem]{Lemma}
\newtheorem{proposition}[theorem]{Proposition}
\newtheorem{definition}{Definition}

\newtheorem{fact}[theorem]{Fact}

\newtheorem{disc}[theorem]{Discussion}

\newcommand{\proofsout}[1]{}
\newcommand{\commentout}[1]{}
\newcommand{\journal}[1]{}

\def\nat{{\mathbb N}}

\newcommand{\removed}[1]{}

\newcommand{\tuple}[1]{\langle #1 \rangle}

\newcommand{\mdp}[1]{\langle #1 \rangle}
\newcommand{\states}{S}
\newcommand{\actions}{A}
\newcommand{\transp}{Tr}
\newcommand{\reward}{R}
\newcommand{\state}{s}
\newcommand{\action}{a}
\newcommand{\horiz}{H}
\newcommand{\horizvar}{h}

\newcommand{\probe}{\rho}

\newcommand{\stree}{\cT}
\newcommand{\initstate}{\state_{0}}

\newcommand{\MCTREE}{\mathsf{MCTS}}
\newcommand{\MCTREEnew}{\mathsf{MCTS2e}}

\newcommand{\BRUSH}{\mathsf{BRUE}}
\newcommand{\BRUSHa}{\mathsf{BRUE}(\alpha)}
\newcommand{\UCT}{\mathsf{UCT}}
\newcommand{\GCT}{\epsilon\mbox{-}\mathsf{greedy+UCT}}
\newcommand{\naiveuniform}{\mathsf{NaiveUniform}}
\newcommand{\craftyuniform}{\mathsf{CraftyUniform}}

\newcommand{\initialization}{search tree $\stree$ $\leftarrow$ root node $\initstate$}
\newcommand{\genprobe}{\mathsf{sample}}
\newcommand{\expandtree}{\mathsf{expand\mbox{-}tree}}
\newcommand{\updatestat}{\mathsf{update\mbox{-}statistics}}
\newcommand{\recommend}{\mathsf{recommend\mbox{-}action}}

\newcommand{\ucbone}{{\sf UCB1}}
\newcommand{\pcount}{n}
\newcommand{\qfunc}{Q}
\newcommand{\qcount}{\widehat{\qfunc}}

\newcommand{\expectation}{{\mathbb{E}}}

\newcommand{\minprob}{\minp}
\newcommand{\mindelta}{\minreg}
\newcommand{\policy}{\pi}
\newcommand{\optpolicy}{\policy^{\ast}}
\newcommand{\brushpolicy}[1]{\policy^{\text{B}}_{#1}}

\newcommand{\range}[1]{\llbracket #1 \rrbracket}

\newcommand{\Optq}{\qfunc}

\newcommand{\Estq}{\qcount}

\newcommand{\optq}[2]{\qfunc_{#2}(#1)}
\newcommand{\estq}[2]{\qcount_{#2}(#1)}

\newcommand{\optqS}[2]{\qfunc_{#2}}
\newcommand{\estqS}[2]{\qcount_{#2}}

\newcommand{\regret}[2]{\Delta_{#2}[#1]}
\newcommand{\sregret}[3]{\regret{#2,#3}{#1}}
\newcommand{\nsa}[2]{n\left(#1\right)}
\newcommand{\minreg}{d}
\newcommand{\minp}{p}

\newcommand{\reca}[1]{\brushpolicy{#1}}

\newcommand{\nums}[2]{n_{#2}(#1)}
\newcommand{\numsS}[2]{n_{#2}}

\newcommand{\paramgenone}{c}
\newcommand{\paramgentwo}{c'}
\newcommand{\paramone}[1]{c_{#1}}
\newcommand{\paramtwo}[1]{c_{#1}'}

\newcommand{\be}{\begin{enumerate}}
\newcommand{\ee}{\end{enumerate}}
\newcommand{\bi}{\begin{itemize}}
\newcommand{\ei}{\end{itemize}}
\newcommand{\bc}{\begin{center}}
\newcommand{\ec}{\end{center}}

\def\beq{\begin{equation}}
\def\eeq{\end{equation}}

\newcommand{\simpleprob}{{\mathbb{P}}}

\newcommand{\slist}{\cL}
\newcommand{\spfunc}{\sigma}
\newcommand{\fake}[1]{\hat{#1}}

\newcommand{\negbinomial}{\text{NB}}
\newcommand{\binomial}{\text{Bin}}
\newcommand{\geomdist}{\text{Geo}}
\newcommand{\bernulli}{\text{Ber}}

\title{Simple Regret Optimization in Online Planning\\ for Markov Decision Processes}
\author{\name Zohar Feldman
\email zoharf@tx.technion.ac.il\\
\name Carmel Domshlak
\email dcarmel@ie.technion.ac.il\\
\addr Faculty of Industrial Engineering \& Management,\\
Technion - Israel Institute of Technology,\\
Haifa, Israel}
\begin{document}

\maketitle

\begin{abstract}
We consider online planning in Markov decision processes (MDPs).
In online planning, the agent focuses on its current state only, deliberates about the set of possible policies from that state onwards and, when interrupted, uses the outcome of that exploratory deliberation to choose what action to perform next.
The performance of algorithms for online planning is assessed in terms of {\em simple regret}, which is the agent's expected performance loss when the chosen action, rather than an optimal one, is followed.

To date, state-of-the-art algorithms for online planning in general MDPs are either best effort, or guarantee only polynomial-rate reduction of simple regret over time.
Here we introduce a new Monte-Carlo tree search algorithm, $\BRUSH$, that guarantees {\em exponential-rate} reduction of simple regret and error probability. 
This algorithm is based on a simple yet non-standard state-space sampling scheme, $\MCTREEnew$, in which different parts of each sample are dedicated to different exploratory objectives. Our  empirical evaluation shows that $\BRUSH$ not only provides superior performance guarantees, but is also very effective in practice and favorably compares to state-of-the-art. We then extend $\BRUSH$ with a variant of ``learning by forgetting.'' The resulting set of algorithms, $\BRUSHa$, generalizes $\BRUSH$, improves  the exponential factor in the upper bound on its reduction rate, and exhibits even more attractive empirical performance.

\end{abstract}


\section{Introduction}
Markov decision processes (MDPs) are a standard model for planning under uncertainty~\cite{puterman:book}.   An MDP $\mdp{\states,\actions,\transp,\reward}$ is defined by a set of possible agent states $\states$, a set of agent actions $\actions$, a stochastic transition function $\transp : \states \times \actions \times \states \rightarrow [0,1]$, and a reward function $\reward : \states \times \actions \times \states \rightarrow {\mathbb{R}}$. 
Depending on the problem domain and the representation language, the description of the MDP can be either declarative or generative (or mixed). 
In any case, the description of the MDP is assumed to be concise. 
While declarative models provide the agents with greater algorithmic flexibility, generative models are more expressive, and both types of models allow for simulated 
execution of all feasible action sequences, from any state of the MDP.
The current state of the agent is fully observable, and the objective of the agent is to act so to maximize its accumulated reward. In the finite horizon setting that will be used for most of the paper, the reward is accumulated over some predefined number of steps $\horiz$.

The desire to handle MDPs with state spaces of size exponential in the  size of the model description  has led researchers to consider online planning in MDPs.
In online planning, the agent, rather than computing a quality policy for the entire MDP before taking any action, focuses only on what action to perform next.
The decision process  consists of a deliberation phase, aka planning, 
terminated  either according to a predefined schedule or due to an external interrupt, and followed by a recommended action for 
the current state. Once that action is applied in the real environment, the decision process is repeated from the obtained state to select the next action and so on.

The quality of the action $a$, recommended for state $s$ with $\horiz$ steps-to-go,
is assessed in terms of the probability that $a$ is sub-optimal, and in terms of
the (closely related) measure of {\em simple regret} $\sregret{\horiz}{\state}{\action}$. The latter captures the performance loss that results from taking $\action$ and then following an optimal policy $\optpolicy$ for the remaining $\horiz-1$ steps, instead of following $\optpolicy$ from the beginning~\cite{bubeck:munos:colt10}. That is,
$$\sregret{\horiz}{\state}{\action} = \qfunc_{\horiz}(\state,\optpolicy(\state,\horiz)) - \qfunc_{\horiz}(\state,\action),$$
where $$\qfunc_{\horiz}(\state,\action) = \expectation_{s'}\left[
\reward(s,a,s') + \qfunc_{\horiz-1}(\state',\optpolicy(\state',\horiz-1))
\right].$$

With a few recent exceptions developed for declarative MDPs~\cite{bonet:geffner:aostar:aaai12,kolobov:etal:aaai12,munos:aistat12}, most algorithms for online MDP planning constitute variants of what is called Monte-Carlo tree search (MCTS). 
One of the earliest and best-known MCTS algorithms for MDPs is the sparse sampling algorithm by Kearns, Mansour, and Ng~\cite{sparsesampling}. Sparse sampling offers a near-optimal action selection in discounted MDPs by constructing a sampled lookahead tree in time exponential in discount factor and suboptimality bound, but independent of the state space size. However, if terminated before an action has proved to be near-optimal, sparse sampling offers no quality guarantees on its action selection. Thus it does not really fit the setup of online planning. Several later works introduced interruptible, anytime MCTS algorithms for MDPs, with $\UCT$~\cite{uct} probably being the most widely used such algorithm these days. Anytime MCTS algorithms are designed to provide convergence to the best action if enough time is given for deliberation, as well as a gradual reduction of performance loss over the deliberation time~\cite{rl:book,PeretG:ecai04,uct,CoquelinM:uai07,cazenave:ijcai09,rosin:ijcai11,tolpin:shimony:aaai12}.
While $\UCT$ and its successors have been devised specifically for MDPs, some of these algorithms are also successfully used in partially observable and adversarial settings~\cite{GellyS11,Sturtevant:ccg08,BjarnasonFT09,BallaF09,EyerichKH10}. 

In general, the relative empirical attractiveness of the various MCTS planning algorithms depends on the
specifics of the problem at hand and cannot usually be predicted ahead of time. When it comes to formal guarantees on the expected performance improvement over the planning time, very few of these algorithms provide such guarantees for general MDPs, and none 
breaks the barrier of the worst-case only {\em polynomial-rate reduction} of simple regret and choice-error probability over time.

This is precisely our contribution here. We introduce a new Monte-Carlo tree search algorithm, $\BRUSH$, that guarantees {\em exponential-rate} reduction of both simple regret and choice-error probability over time, for general MDPs over finite state spaces. The algorithm is based on a simple and efficiently implementable sampling scheme, $\MCTREEnew$, in which different parts of each sample are dedicated to different competing exploratory objectives. The motivation for this objective decoupling came from  a recently growing understanding that the current MCTS algorithms for MDPs do not optimize the reduction of simple regret directly, but only  via optimizing what is called cumulative regret, a performance measure suitable for the (very different) setting of reinforcement learning~\cite{bubeck:munos:colt10,munos:aistat12,tolpin:shimony:aaai12,feldman:domshlak:arxiv12}. 
Our  empirical evaluation on some standard MDP benchmarks for comparison between MCTS  planning algorithms shows that $\BRUSH$ not only provides superior performance guarantees, but is also very effective in practice and favorably compares to state of the art. We then extend $\BRUSH$ with a variant of ``learning by forgetting.'' The resulting family of algorithms, $\BRUSHa$, generalizes $\BRUSH$, improves  the exponential factor in the upper bound on its reduction rate, and exhibits even more attractive empirical performance.

\section{Monte-Carlo Planning}

$\MCTREE$, a high-level scheme for Monte-Carlo tree search that gives rise to various specific algorithms for online MDP planning, is depicted in Figure~\ref{fig:treesearch}. Starting with the current state $\initstate$, $\MCTREE$ performs an iterative construction of a tree $\stree$ rooted at $\initstate$. At each iteration, $\MCTREE$ issues a state-space sample from $\initstate$, expands the tree $\stree$ using the outcome of that sample, and updates  information stored at the nodes of $\stree$. Once the simulation phase is over, $\MCTREE$ uses the information collected at the nodes of $\stree$ to recommend an action to perform in $\initstate$. 
For compatibility of the notation with prior literature,
in what follows we refer to the tree nodes via the states associated with these nodes. 
Note that, due to the Markovian nature of MDPs, it is unreasonable to distinguish between  nodes associated with the same state at the same depth. Hence, the actual graph constructed by most instances of $\MCTREE$ forms a DAG over nodes $(\state,\horizvar) \in \states\times \{0,1,\dots,\horiz\}$. 
By $A(s) \subseteq A$ in what follows, we refer to the subset of actions applicable in state $s$.

\begin{figure}
\begin{center}
\begin{tabular}{c}
\hline
\begin{minipage}{3in}
\vspace{1mm}
\begin{tabbing}
\underline{$\MCTREE$}: [{input}: \= $\mdp{\states,\actions,\transp,\reward}$; $\initstate \in \states$]\\
\initialization\\
{\bf while} time permits:\\
\> $\probe \leftarrow \genprobe(\initstate,\stree)$\\
\> $\stree \leftarrow  \expandtree(\stree,\probe)$\\
\> $\updatestat(\stree,\probe)$\\
{\bf return} $\recommend(\initstate,\stree)$
\end{tabbing}
\end{minipage}\\
\hline
\end{tabular}
\end{center}
\caption{\label{fig:treesearch} High-level scheme for regular Monte-Carlo tree sampling.}
\end{figure}

Numerous concrete instances of $\MCTREE$ have been proposed, with $\UCT$~\cite{uct}
probably being the most popular such algorithm these days~\cite{GellyS11,Sturtevant:ccg08,BjarnasonFT09,BallaF09,EyerichKH10,KellerE12}. To give a concrete sense of $\MCTREE$'s components, as well as to ground some intuitions discussed later on, below we describe the specific setting of $\MCTREE$ corresponding to the core $\UCT$ algorithm, and Figure~\ref{fig:uctdynamics} illustrates the $\UCT$ tree construction, with $n$ denoting the number of state-space samples.
\begin{itemize}
\item $\genprobe$: 
The samples $\probe = \tuple{s_{0},a_{1},s_{1},\dots,a_{k},s_{k}}$ are
all issued from the root node $\initstate$. The sample ends either when a sink state is reached, that is, $A(s_{k})=\emptyset$, or when $k=H$.
Each node/action pair $(\state,\action)$ is 
associated with a counter $\pcount(\state,\action)$ and a value accumulator $\qcount(\state,\action)$. Both $\pcount(\state,\action)$ and $\qcount(\state,\action)$ are  initialized to $0$, and then updated by the $\updatestat$ procedure. Given $s_{i}$, the next-on-the-sample action $a_{i+1}$ is selected according to the deterministic $\ucbone$ policy~\cite{auer:etal:ml02}, originally proposed for optimal cumulative regret minimization in stochastic multi-armed bandit (MAB) problems~\cite{robbins:52}:
If $n(s_{i},a) > 0$ for all $a \in \actions(s_{i})$, then
\begin{equation}
\label{e:uctselect}
a_{i+1} = \argmax_{a}{\left[\qcount(s_{i},a) + c\sqrt{\frac{\log{\pcount(s_{i})}}{\pcount(s_{i},a)}}\right]},
\end{equation}
where $\pcount(s) = \sum_{a}{\pcount(\state,\action)}$. Otherwise, $a_{i+1}$ is selected uniformly at random from the still unexplored actions $\{a \in \actions(s_{i}) \mid n(s_{i},a) = 0\}$. 
In both cases, $s_{i+1}$ is then sampled according to the conditional probability ${\mathbb{P}}(S | s_{i},a_{i+1})$, induced by the transition function $\transp$.
\item $\expandtree$: Each state-space sample $\probe = \tuple{s_{0},a_{1},s_{1},\dots,a_{k},s_{k}}$ induces a state trace $\tuple{s_{0},s_{1},\dots,s_{i}}$ inside $\stree$, as well as a state trace $\tuple{s_{i+1},\dots,s_{k}}$ outside of $\stree$. In principle, $\stree$ can be expanded with any prefix of $\tuple{s_{i+1},\dots,s_{k}}$; a popular choice in prior work appears to be expanding $\stree$ with only the upper-most node $s_{i+1}$. (If $\stree$ is constructed as a DAG, it is expanded with the first node along $\probe$ that leaves $\stree$.)
\item $\updatestat$: For each node $s_{i}$ along $\probe$ that is now part of the expanded tree $\stree$,
the counter $\pcount(\state_{i},\action_{i+1})$ is incremented and the estimated $Q$-value is updated as
{
\begin{equation}
\label{e:uctupdate}
\begin{split}
\qcount(\state_{i},\action_{i+1}) \leftarrow \;\; &
\qcount(\state_{i},\action_{i+1})+ \frac{\reward_{i}-\qcount(\state_{i},\action_{i+1})}{\pcount(\state_{i},\action_{i+1})},
\end{split}
\end{equation}
}
where $\reward_{i} = \sum_{j=i}^{k-1}\reward(s_{j},a_{j+1},s_{j+1})$.
\item $\recommend$: Interestingly, the action recommendation protocol of $\UCT$ was never properly specified, and different applications of $\UCT$ adopt different decision rules, including maximization of the estimated $Q$-value, of the augmented estimated $Q$-value as in Eq.~\ref{e:uctselect}, of the number of times the action was selected during the simulation, as well as randomized protocols based on the information collected at the root.
\end{itemize}

\begin{figure}[t]
\begin{center}
\begin{tabular}{|c|c|c|c|c|}
\hline
\imagetop{\includegraphics[width=2cm]{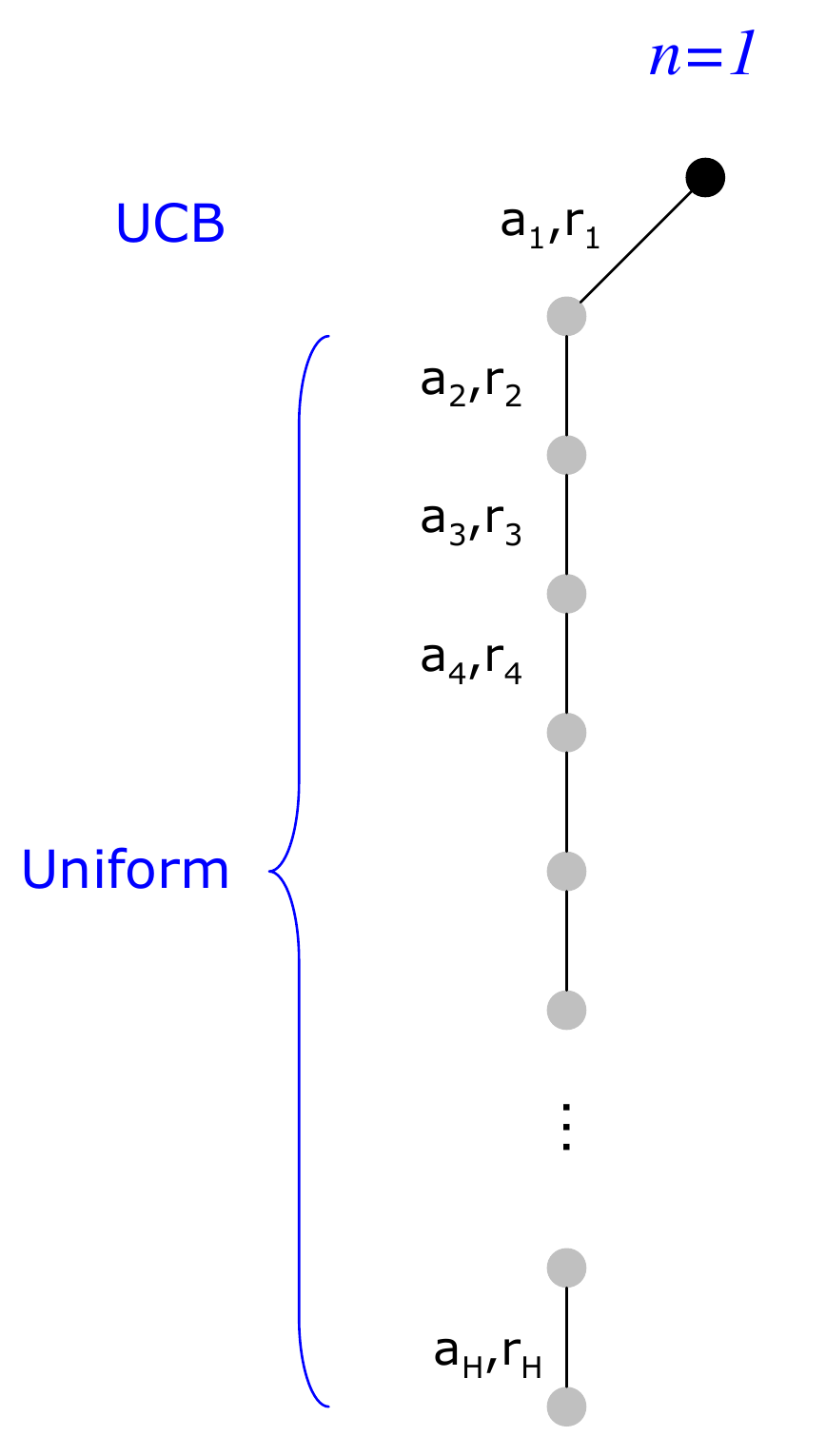}}
&
\imagetop{\includegraphics[width=1.7cm]{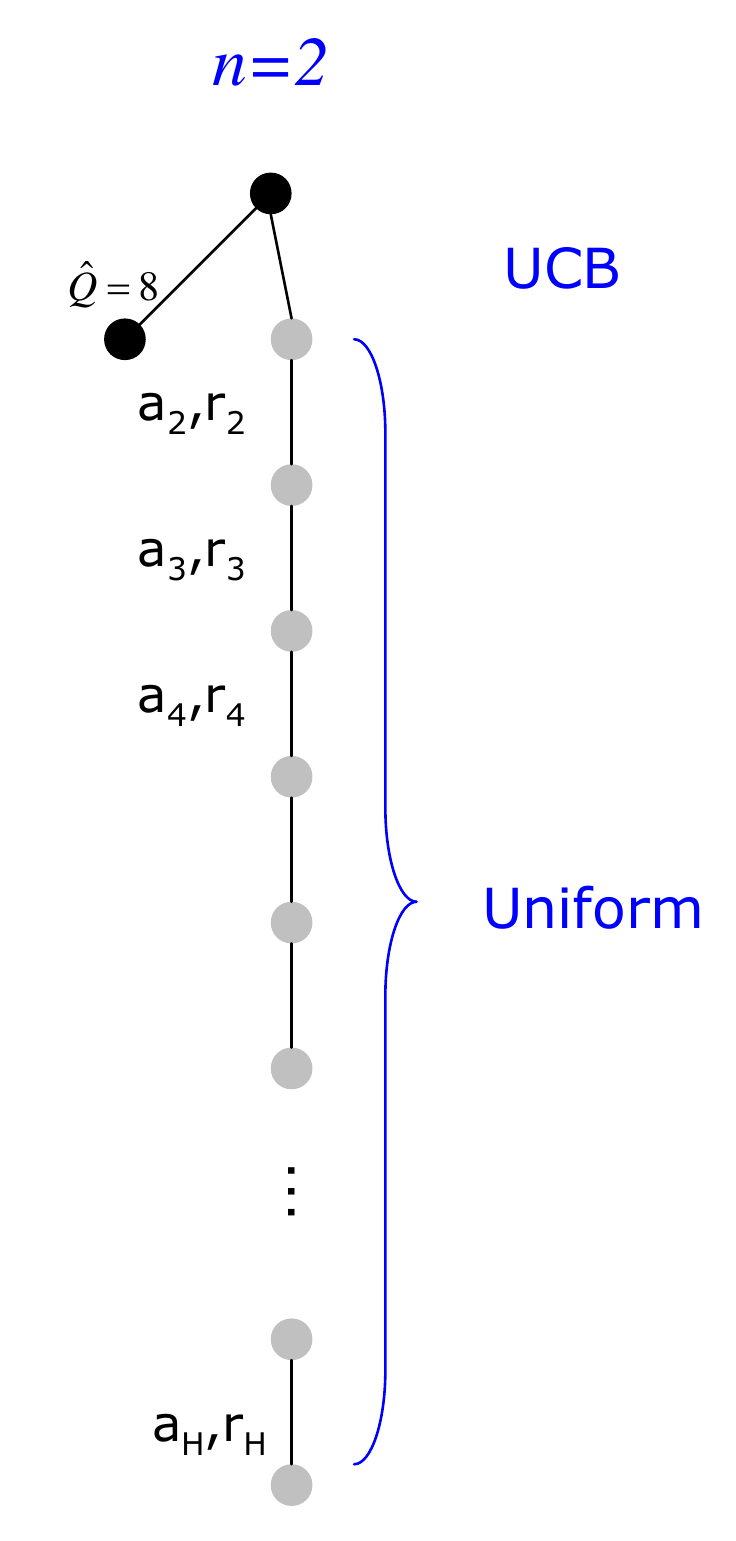}}
&
\imagetop{\includegraphics[width=2.2cm]{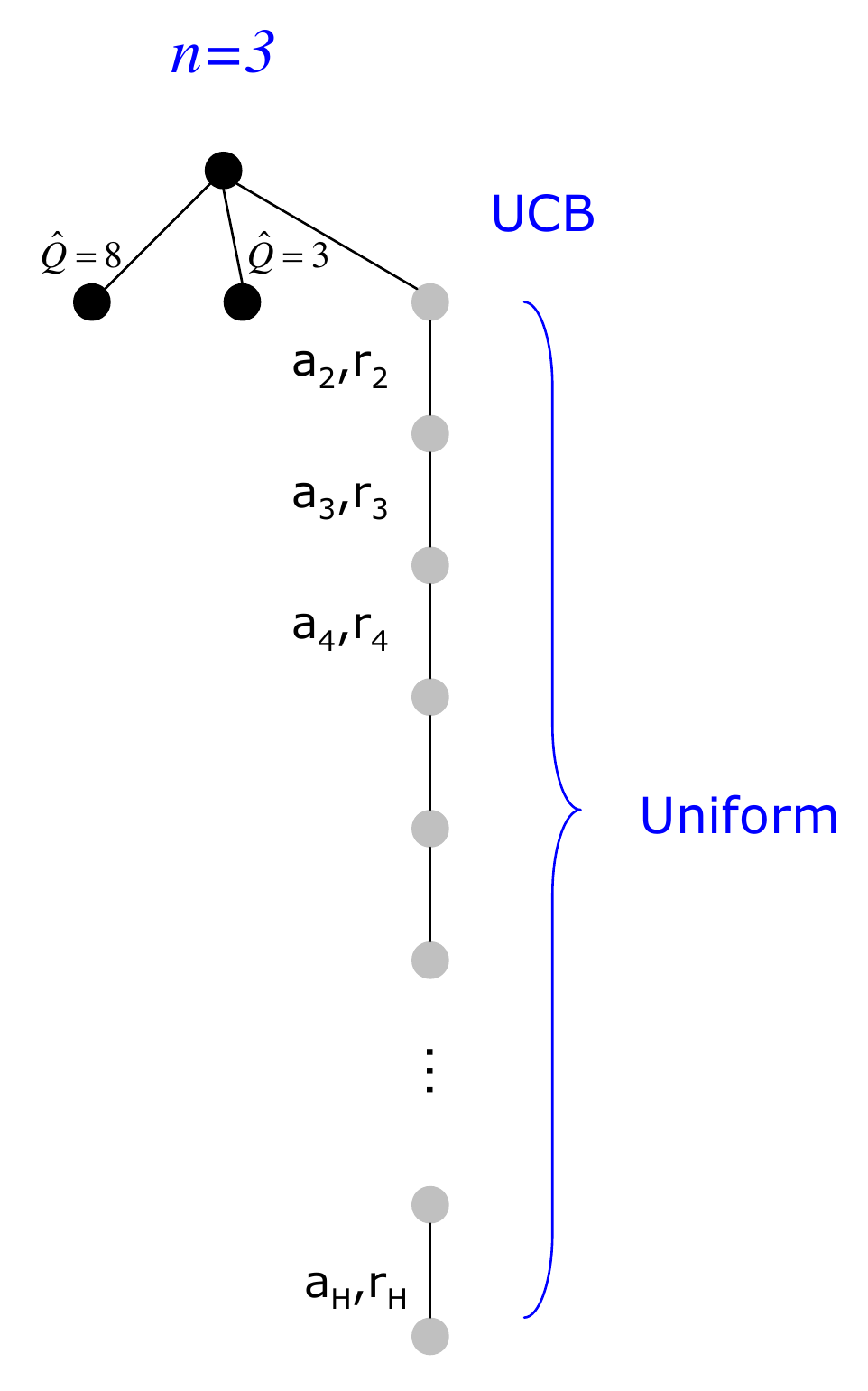}}
&
\imagetop{\includegraphics[width=2.3cm]{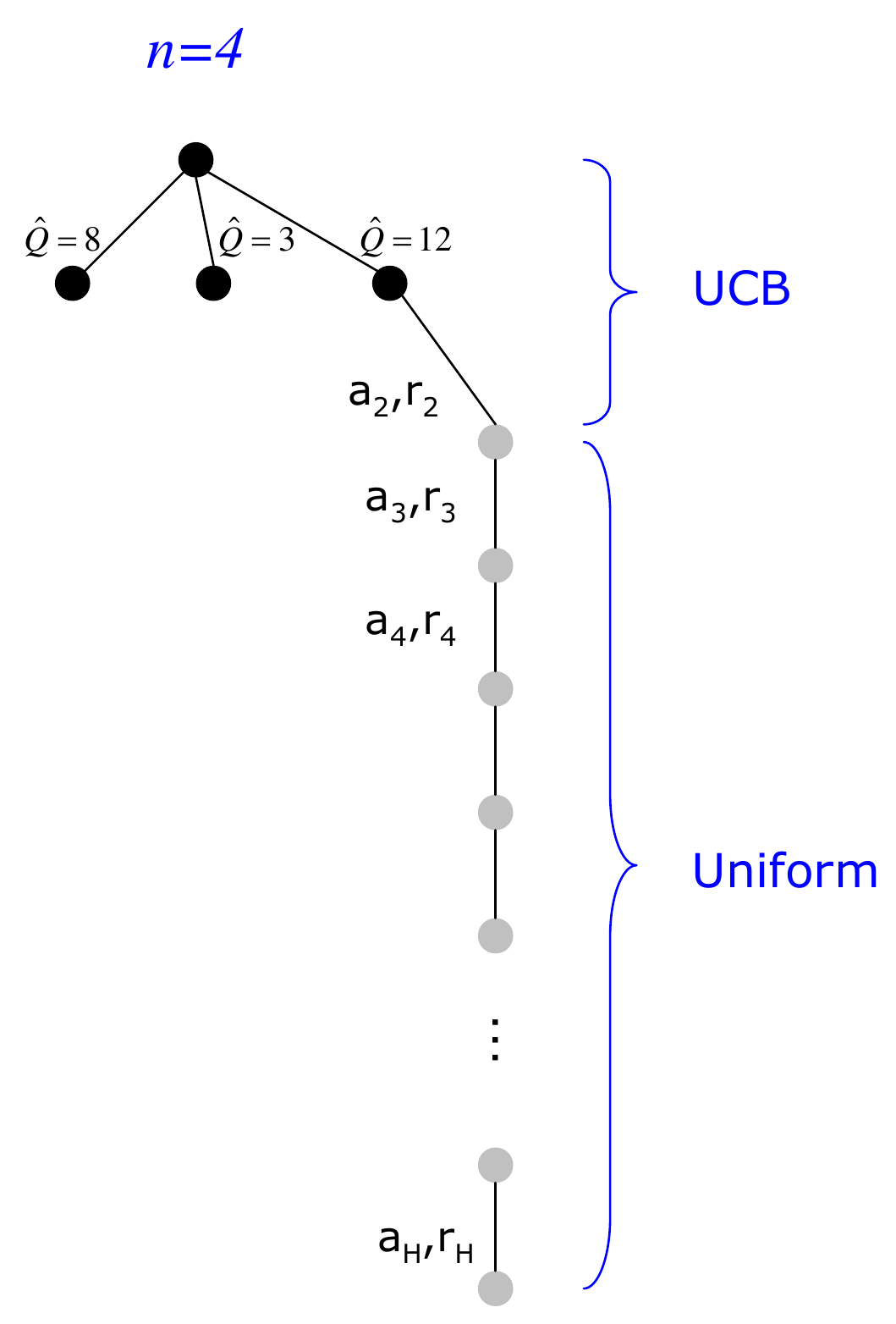}}
&
\imagetop{\includegraphics[width=2.6cm]{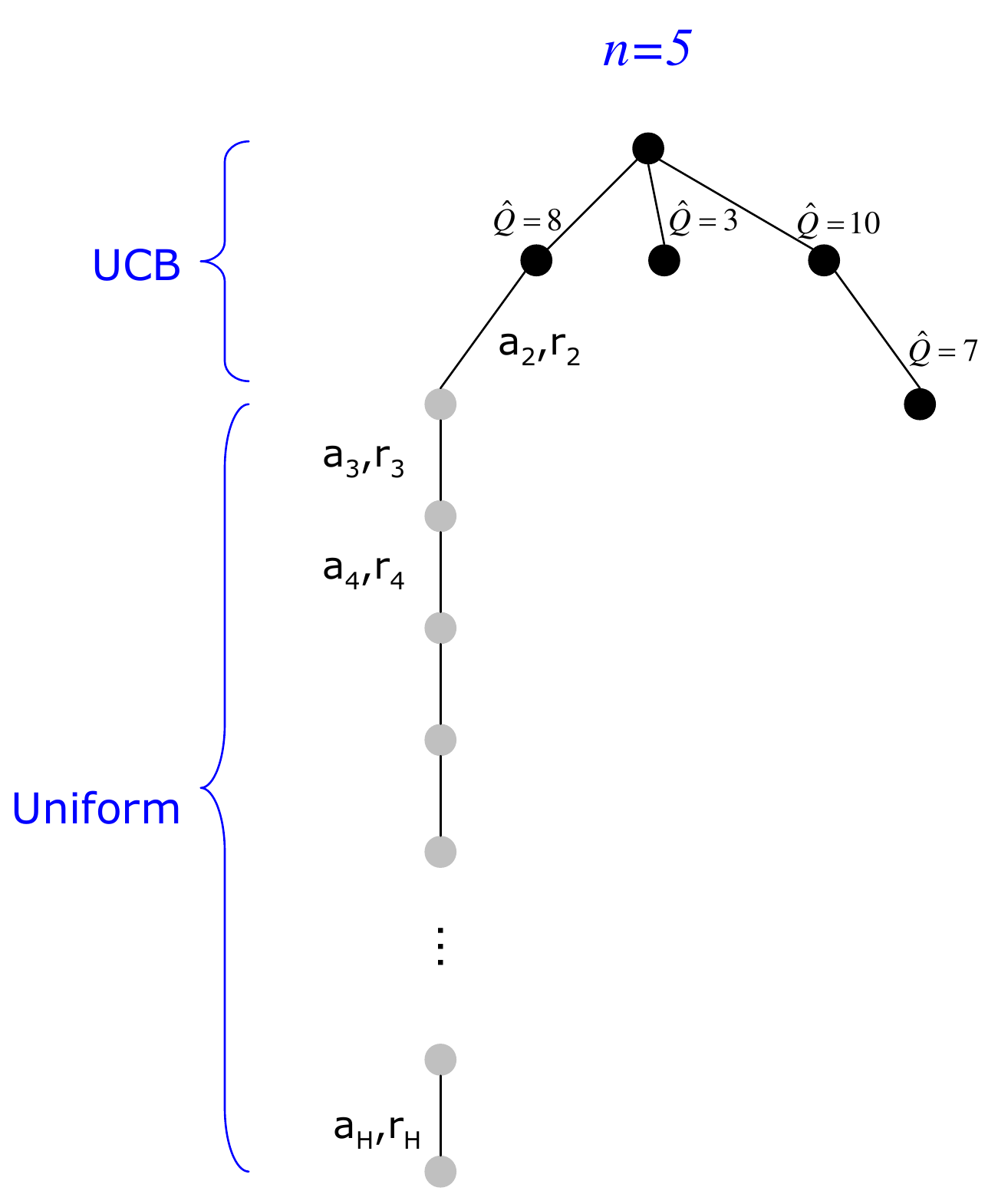}}
\\
\hline
\imagetop{\includegraphics[width=1.5cm]{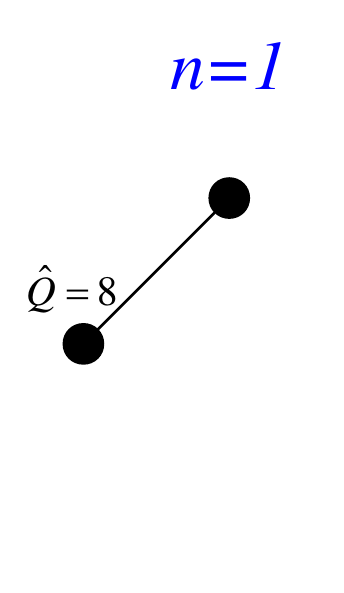}}
&
\imagetop{\includegraphics[width=1.5cm]{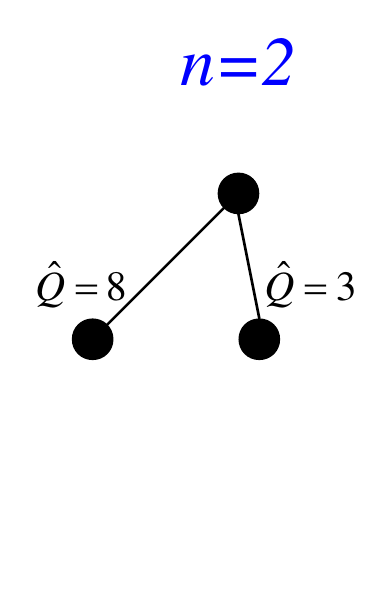}}
&
\imagetop{\includegraphics[width=2.2cm]{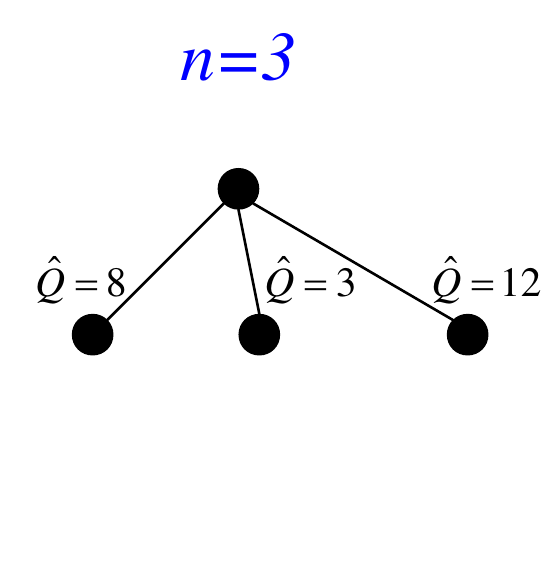}}
&
\imagetop{\includegraphics[width=2.6cm]{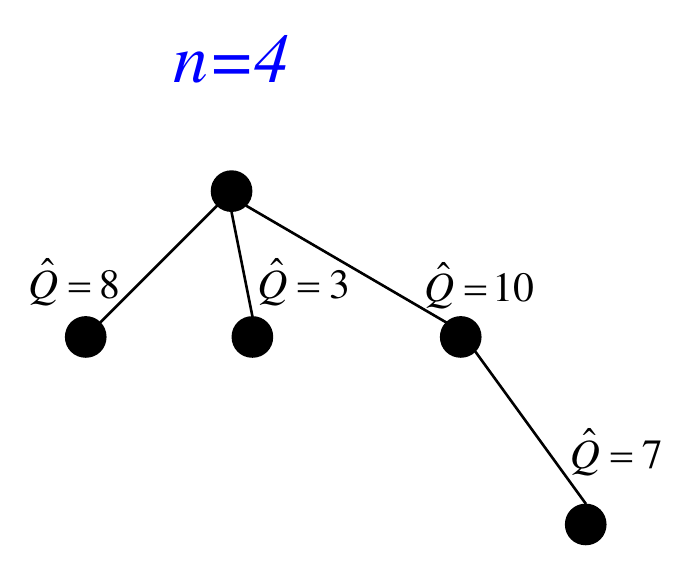}}
&
\imagetop{\includegraphics[width=3cm]{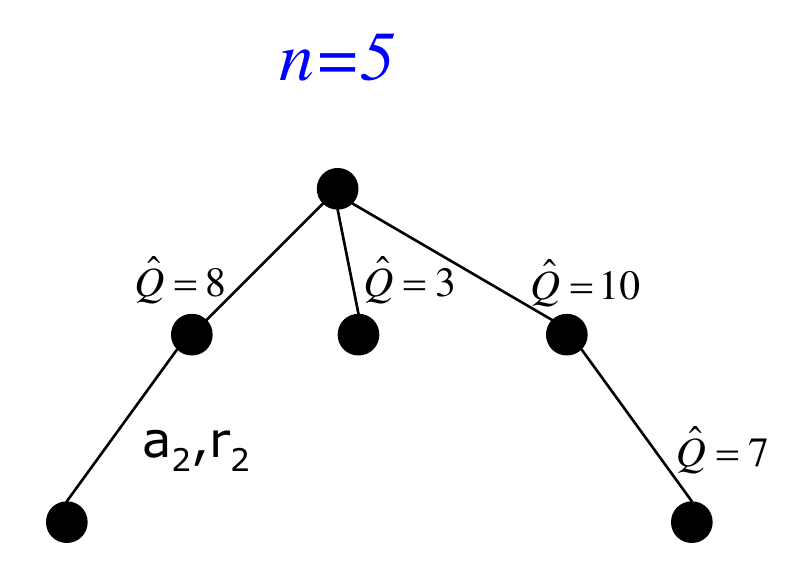}}
\\
\hline
\multicolumn{5}{|c|}{
\imagetop{\includegraphics[width=4cm]{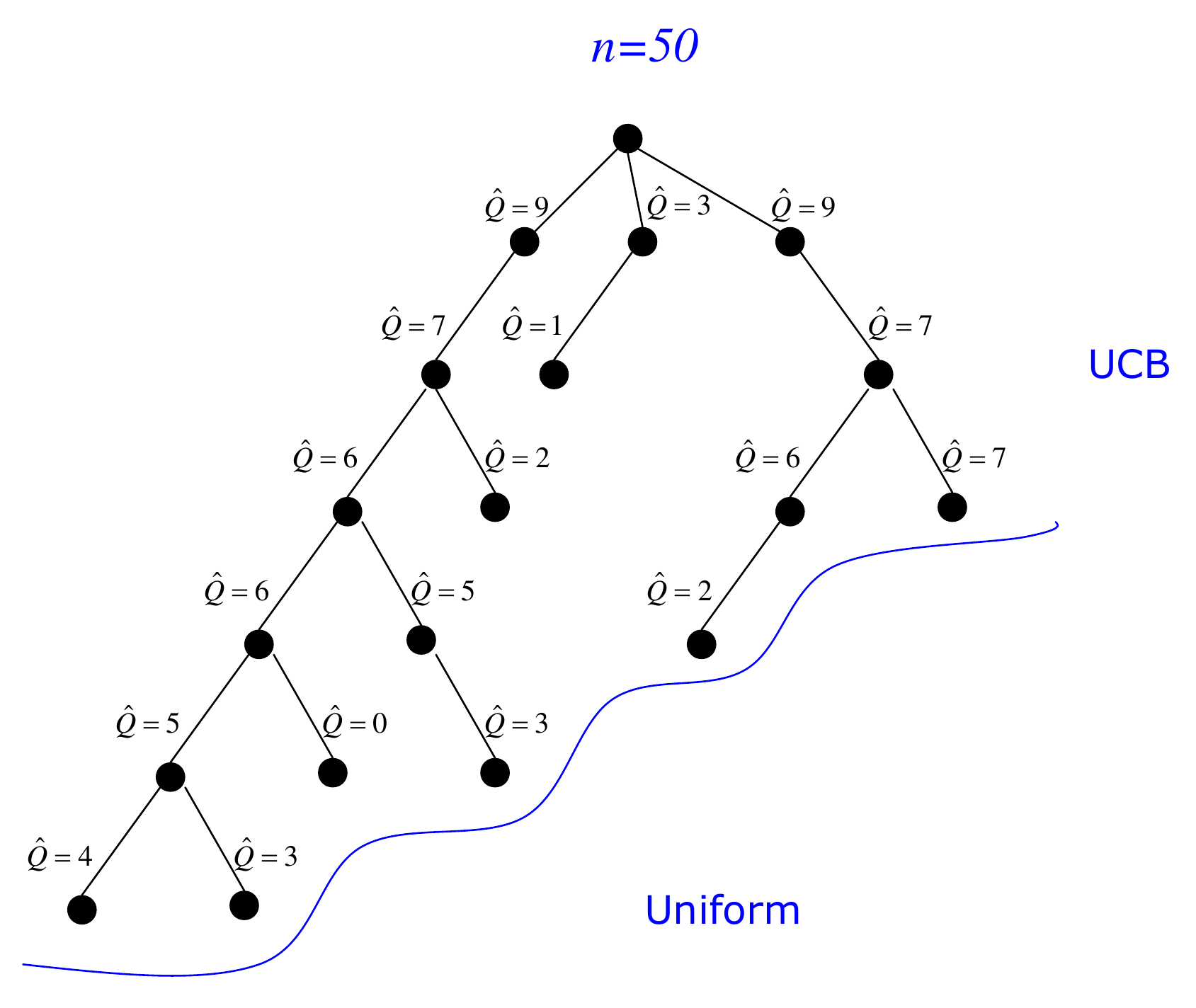}}
}\\
\hline
\end{tabular}
\end{center}
\caption{\label{fig:uctdynamics}Illustration of the $\UCT$ dynamics}
\end{figure}

The key property of $\UCT$ 
is that its exploration of the search space is obtained by considering a hierarchy of forecasters, each minimizing its own {\em cumulative regret}, that is, the loss of the total reward incurred by exploring the environment~\cite{auer:etal:ml02}. Each such pseudo-agent  forecaster corresponds to a state/steps-to-go pair $(\state,\horizvar)$. In that respect, according to Theorem~6 of~\citeA{uct}, $\UCT$ asymptotically achieves the best possible (logarithmic) cumulative regret. However, as recently pointed out in  numerous works~\cite{bubeck:munos:colt10,munos:aistat12,tolpin:shimony:aaai12,feldman:domshlak:arxiv12},  cumulative regret does not seem to be the right objective for online MDP planning, and this is because the rewards ``collected'' at the simulation phase are fictitious. Furthermore, the work of~\citeA{bubeck:etal:tcs11} on multi-armed bandits shows that minimizing cumulative regret and minimizing simple regret are somewhat competing objectives. Indeed, the same Theorem~6 of~\citeA{uct} claims only a polynomial-rate reduction of the probability of choosing a non-optimal action, and the  results  of~\citeA{bubeck:etal:tcs11} on simple regret minimization in MABs with stochastic rewards imply that $\UCT$ achieves only polynomial-rate reduction of the simple regret over time. Some attempts have recently been made to adapt $\UCT$, and $\MCTREE$-based planning in general, to  optimizing simple regret in online MDP planning directly, and some of these attempts were  empirically rather successful~\cite{tolpin:shimony:aaai12,hay:etal:uai12}. However, to the best of our knowledge, none of them breaks  $\UCT$'s barrier of the worst-case polynomial-rate reduction of simple regret over time.
%

\section{Simple Regret Minimization in MDPs}
We now show that exponential-rate reduction of simple regret in online MDP planning is achievable. To do so, we first motivate and introduce a family of $\MCTREE$ algorithms with a two-phase scheme for generating state space samples, and then describe a concrete  algorithm from this family, $\BRUSH$, that (1) guarantees that the probability of recommending a non-optimal action asymptotically convergences to zero at an exponential rate, and (2) achieves exponential-rate reduction of  simple regret over time.

\subsection{Exploratory concerns in online MDP planning}

The work of~\citeA{bubeck:etal:tcs11} on pure exploration in multi-armed bandit (MAB) problems was probably the first to stress that the minimal simple regret  can increase as the bound on the cumulative regret is decreases. At a high level,~\citeA{bubeck:etal:tcs11} show that efficient schemes for simple regret minimization in MAB should be as exploratory as possible, thus improving the expected quality of the recommendation issued at the end of the learning process. In particular, they showed that the simple round-robin sampling of MAB actions, followed by recommending the action with the highest empirical mean, yields exponential-rate reduction of simple regret, while the $\ucbone$ strategy that balances between exploration and exploitation yields  only polynomial-rate reduction of that measure. In that respect, the situation with MDPs is seemingly no different, and thus Monte-Carlo MDP planning should focus on exploration only. However, the answer to the question of what it means to be ``as exploratory as possible'' with MDPs is less straightforward than it is in  the special case of MABs.

For an intuition as to why the ``pure exploration dilemma'' in MDPs is somewhat complicated, 
consider the state/steps-to-go pairs $(s,\horizvar)$ as pseudo-agents, all acting on behalf of the root pseudo-agent $(\initstate,\horiz)$ that aims at  minimizing  its own simple regret in a stochastic MAB induced by the applicable  actions $\actions(\initstate)$. Clearly, if an oracle would provide $(\initstate,\horiz)$ with an optimal action $\optpolicy(\initstate,\horiz)$, then no further deliberation would be needed until after the execution of $\optpolicy(\initstate,\horiz)$. However, the task characteristics of $(\initstate,\horiz)$ are an exception rather than a rule. Suppose that an oracle provides us with  optimal actions for {\em all} pseudo-agents $(\state,\horizvar)$ {\em but} $(\initstate,\horiz)$. Despite the richness of this information, $(\initstate,\horiz)$ in some sense remains as clueless as it was before: To choose between the actions in $\actions(\initstate)$, $(\initstate,\horiz)$ needs, at the very least,  some ordinal information about the expected value of these alternatives. Hence, when sampling the futures, each non-root pseudo-agent $(\state,\horizvar)$ should be devoted to {\em two}  objectives:
\begin{enumerate}[(1)]
\item identifying  an optimal action $\optpolicy(\state,\horizvar)$, and
\item estimating the actual value of that action, because this information is needed by the {predecessor(s)} of $(\state,\horizvar)$ in $\stree$.
\end{enumerate}

Note that both these objectives are {\em exploratory}, yet the problem is that they are somewhat competing. In that respect, the choices made by $\UCT$ actually make sense:
Each sample $\probe$ issued by $\UCT$ at $(\state,\horizvar)$ is a priori devoted {\em both} to increasing the confidence in that some current candidate $\action^{\dagger}$ for $\optpolicy(\state,\horizvar)$ is indeed $\optpolicy(\state,\horizvar)$, as well as to improving the estimate of $\qfunc_{\horizvar}(\state,\action^{\dagger})$,
while as if assuming that $\optpolicy(\state,\horizvar)=\action^{\dagger}$. However, while such an overloading of the samples is unavoidable in the ``learning while acting'' setup of reinforcement learning, this should not necessarily be the case in online planning. Moreover, this sample overloading in $\UCT$ comes with a high price: As it was shown by~\citeA{CoquelinM:uai07}, the number of samples after which the bounds of $\UCT$ on both simple and cumulative regret become meaningful might be as high as hyper-exponential in $H$.

\subsection{Separation of Concerns at the Extreme}

Separating the two aforementioned exploratory concerns is at the focus of our investigation here. Let $\initstate$ be a state of an MDP $\mdp{\states,\actions,\transp,\reward}$ with rewards in $\left[0,1\right]$,  $K$ applicable actions at each state, $B$ possible outcome states for each action, and finite horizon $H$. First, to get a sense of what separation of exploratory concerns in online planning can buy us, we begin with a MAB perspective on MDPs, with each arm in the MAB corresponding to a ``flat" policy of acting for $H$ steps starting from the current state $\initstate$.  A ``flat" policy $\pi$ is a minimal partial mapping from state/steps-to-go pairs to actions that fully specifies an acting strategy in the MDP for $H$ steps, starting at $\initstate$. Sampling such an arm $\pi$ is straightforward as $\pi$ prescribes precisely which action should be applied at every state that can possibly be encountered along the execution of $\pi$. The reward of such an arm $\pi$ is stochastic, with support $\left[0,H\right]$, and the number of arms in this schematic MAB is $K'=K^{\sum_{i=0}^{H-1}B^i} \approx K^{B^H}$.

Now, consider a simple algorithm,  $\naiveuniform$,  which systematically samples each "flat" policy in a loop, and updates the estimation of the corresponding arm with the obtained reward. If stopped at iteration $n$, the algorithm recommends $\pi(s_0)$, where $\pi$ is the arm/policy with best empirical value $\hat{\mu}_{\pi,n}$.  
By the iteration $n$ of this algorithm, each arm will be sampled at least $\lfloor \frac{n}{K^{B^H}}\rfloor$ times. Therefore, using the Hoeffding's  inequality, the probability that the chosen arm $\pi$ is sub-optimal in our MAB  is bounded by 
\begin{equation}
\mathbb{P}\left\{\hat{\mu}_{\pi,n}>\hat{\mu}_{\pi^{\ast},n}\right\}=\mathbb{P}\left\{ \hat{\mu}_{\pi,n}-\hat{\mu}_{\pi^{\ast},n}-\left(-\Delta_{\pi}\right)\geq\Delta_{\pi}\right\}
\leq \exp\left(-\frac{\lfloor\frac{n}{K^{B^H}}\rfloor\Delta_{\pi}^2}{2H^2}\right),
\end{equation}
where $\Delta_{\pi} = \mu_{\pi^{\ast}}-\mu_{\pi}$, and thus the expected simple regret can be bounded as 
\begin{equation}\label{eq:naive_regret}
\mathbb{E}r_n\leq HK^{B^H} \exp\left(-\frac{\lfloor\frac{n}{K^{B^H}}\rfloor d^2}{2H^2}\right).
\end{equation}

Note that $\naiveuniform$  uses each sample $\rho= \left(\initstate,a_0,s_1,a_1,\ldots,a_{H-1},s_H\right)$ to update the estimation of only a single policy $\pi$. However, recalling that arms in our MAB problem are actually compound policies, the same sample can in principle be used to update the estimates of all  policies $\pi'$ that are consistent with $\rho$ in the sense that, for $0 \leq i \leq H-1$, 
$\pi'(s_i,H-i)$ is defined and it is defined as $\pi'(s_i,H-i) = a_{i}$. 
The resulting 
 algorithm,  $\craftyuniform$, 
generates samples by choosing the actions along them uniformly at random, 
and uses the outcome of each sample to update all the policies consistent with it. 
Note that sampling the arms in $\craftyuniform$ cannot be done systematically as in $\naiveuniform$  because the set of policies updated at each iteration is stochastic. 

Since the sampling is uniform, the probability of any policy to be updated by the  sample issued at any iteration of $\craftyuniform$ is $\frac{1}{K^H}$. For an arm $\pi'$, let $N_{\pi',n}$ denote the number of samples issued at the $n$ iterations of $\craftyuniform$ that are consistent with the policy $\pi'$.  The probability that $\pi$, the best empirical arm  after $n$ iterations,  is sub-optimal is bounded by
\begin{equation}
\mathbb{P}\left\{\hat{\mu}_{\pi,n}>\hat{\mu}_{\pi*,n}\right\}\leq\mathbb{P}\left\{ \hat{\mu}_{\pi,n}-\mu_{\pi}\geq\frac{\Delta_{\pi}}{2}\right\} + \mathbb{P}\left\{ \hat{\mu}_{\pi*,n}-\mu_{\pi*}\geq\frac{\Delta_{\pi}}{2}\right\}.
\end{equation}
Each of the two terms on the right-hand side can be bounded as:
\begin{equation}
\begin{split}
\mathbb{P}\left\{ \hat{\mu}_{\pi,n}-\mu_{\pi}\geq\frac{\Delta_{\pi}}{2}\right\}&\leq\mathbb{P}\left\{N_{\pi,n}\leq\frac{n}{2K^H}\right\}+\mathbb{P}\left\{N_{\pi,n}>\frac{n}{2K^H},\;\hat{\mu}_{\pi,n}-\mu_{\pi}\geq\frac{\Delta_{\pi}}{2}\right\}\\
&\stackrel{(\dagger)}{\leq} e^{-\frac{n}{2K^{2H}}}+\sum_{i=\frac{n}{2K^H}+1}^n \mathbb{P}\left\{N_{\pi,n}=i\right\}\mathbb{P}\left\{\hat{\mu}_{\pi,n}-\mu_{\pi}\geq\frac{\Delta_{\pi}}{2}\;\middle|\;N_{\pi,n}=i\right\}\\
&\leq e^{-\frac{n}{2K^{2H}}}+\mathbb{P}\left\{\hat{\mu}_{\pi,n}-\mu_{\pi}\geq\frac{\Delta_{\pi}}{2}\;\middle|\;N_{\pi,n}=\frac{n}{2K^H}+1\right\}\sum_{i=\frac{n}{2K^H}+1}^n \mathbb{P}\left\{N_{\pi,n}=i\right\}\\
&\leq e^{-\frac{n}{2K^{2H}}}+\mathbb{P}\left\{\hat{\mu}_{\pi,n}-\mu_{\pi}\geq\frac{\Delta_{\pi}}{2}\;\middle|\;N_{\pi,n}=\frac{n}{2K^H}+1\right\}\\
&\stackrel{(\ddagger)}{\leq} e^{-\frac{n}{2K^{2H}}}+e^{-\frac{n\Delta_{\pi}^2}{4K^{H}H^2}}\\
&\leq 2 e^{-\frac{n\Delta_{\pi}^2}{4K^{2H}H^2}}, 
\end{split}
\end{equation}
where $(\dagger)$ and $(\ddagger)$ are by the Hoeffding inequality. 
In turn, similarly to Eq.~\ref{eq:naive_regret}, the simple regret for $\craftyuniform$ is  bounded by
\begin{equation}\label{eq:crafty_regret}
\mathbb{E}r_n\leq 4HK^{B^H} e^{-\frac{nd^2}{4K^{2H}H^2}}.
\end{equation}
Since $H$ is a trivial upper-bound on $\mathbb{E}r_n$, the bound in Eq.~\ref{eq:crafty_regret} becomes effective only when $4K^{B^H}\exp\left(-\frac{nd^2}{4K^{2H}H^2}\right)<1$, that is, for 
\begin{equation}
\label{eq:transitioncrafty}
n>\left(K^2B\right)^H \cdot 4\left(\frac{H}{d}\right)^2 \log K.
\end{equation} 
Note that this transition period length is still much better than that of $\UCT$, which is hyper-exponential in $H$. Moreover, unlike in $\UCT$, the rate of the simple regret reduction is then exponential in the number of iterations.

\subsection{Two-phase sampling and $\BRUSH$}

While both the simple regret convergence rate, as well as the length of the transition period of  $\craftyuniform$, are more attractive than those of  $\UCT$, this in itself is not much of a help: $\craftyuniform$ requires explicit reasoning about $K^{B^H}$ arms, and thus it cannot be efficiently implemented. However, it does show the promise of  separation of  concerns in online planning. 
We now introduce an $\MCTREE$ family of algorithms, referred to as $\MCTREEnew$, that allows utilizing this promise to a large extent.

The  instances of the $\MCTREEnew$ family vary along four parameters: {\em switching point function} $\spfunc: \nat \rightarrow \{1,\dots,\horiz\}$, {\em exploration policy}, {\em estimation policy}, and {\em update policy}. 
With respect to these four parameters, the $\MCTREE$ components in $\MCTREEnew$ are as follows.
\begin{itemize}
\item Similarly to $\UCT$, each node/action pair $(\state,\action)$ is associated with variables $\pcount(\state,\action)$ and $\qcount(\state,\action)$. However, while counters $\pcount(\state,\action)$ are initialized to $0$, value accumulators $\qcount(\state,\action)$ are schematically initialized to $-\infty$.

\item $\genprobe$: Each iteration of $\BRUSH$ corresponds to a single state-space sample of the MDP, and these samples $\probe = \tuple{s_{0},a_{1},s_{1},\dots,a_{k},s_{k}}$ are all issued from the root node $\initstate$. The sample ends either when  a sink state is reached, that is, $A(s_{k})=\emptyset$, or when $k=H$. The generation of $\probe$ is done in two phases: At iteration $n$, the actions at states $s_{0},\dots,s_{\spfunc(n)-1}$ are selected according to the exploration policy of the algorithm, while the actions at states $s_{\spfunc(n)},\dots,s_{k-1}$ are selected according to its estimation policy.
\item $\expandtree$:  $\stree$ is expanded with the suffix of state sequence  $s_1,\ldots,s_{\spfunc(n)-1}$ that is new to $\stree$. 
\item $\updatestat$: For each state $s_{i} \in \{s_{0},\dots,s_{\spfunc(n)-1}\}$, the update policy of the algorithm prescribes whether it should be updated. If $s_{i}$ should be updated, then 
the counter $\pcount(\state_{i},\action_{i+1})$ is incremented and the estimated $Q$-value is updated according to Eq.~\ref{e:uctupdate} (p.~\pageref{e:uctupdate}).

\item $\recommend$: The recommended action is chosen uniformly at random among the actions $a$ maximizing $\qcount(s_{0},a)$. 
\end{itemize}
In what follows, for $n>0$, the $n$-th iteration of $\BRUSH$ will be called {\em $\cH$-iteration} if $\spfunc(n)=\cH$. At a high  level, the two phases of sample  generation  respectively target the two exploratory objectives of online MDP planning: While the sample prefixes aim at exploring the options, the sample suffixes aim at improving the value estimates for  the current candidates for $\optpolicy$. In particular, this separation allows us to introduce a specific $\MCTREEnew$ instance, $\BRUSH$,\footnote{Short for {\bf B}est {\bf R}ecommendation with {\bf U}niform {\bf E}xploration; the name is carried on from our first presentation of the algorithm in~\cite{feldman:domshlak:arxiv12}, where ``estimation'' was referred to as ``recommendation.''} that is tailored to simple regret minimization.
%
The $\BRUSH$ setting of  $\MCTREEnew$
is described below, and Figure~\ref{fig:bruedynamics} illustrates its dynamics.
\begin{itemize}
\item The switching point function $\spfunc: \nat \rightarrow \{1,\dots,\horiz\}$ is
\begin{equation}
\label{eq:nh1}
\spfunc(n) = H-((n-1) \!\!\mod H),
\end{equation}
that is, the depth of exploration is chosen by a round-robin on $\{1,\dots,H\}$, in reverse order.

\item At state $s$, the exploration policy samples an action uniformly at
random, while the estimation policy samples an action uniformly at
random, but only among the actions $a \in \actions({\state})$ that {\em maximize} $\qcount({\state},\action)$.

\item For a sample $\probe$ issued at iteration $n$, only the state/action pair $(s_{\spfunc(n)-1},a_{\spfunc(n)})$ immediately preceding  the switching state $s_{\spfunc(n)}$ along $\probe$ is updated. That is, the information obtained by the second phase of $\probe$ is used only for improving the estimate at state $s_{\spfunc(n)-1}$, and is {\em not} pushed further up the sample. While that may appear wasteful and even counterintuitive, this locality of update is required to satisfy the formal guarantees of $\BRUSH$ discussed below.
\end{itemize}

\begin{figure}[t]
\begin{center}
\begin{tabular}{|c|c|c|c|}
\hline
\imagetop{\includegraphics[width=3cm]{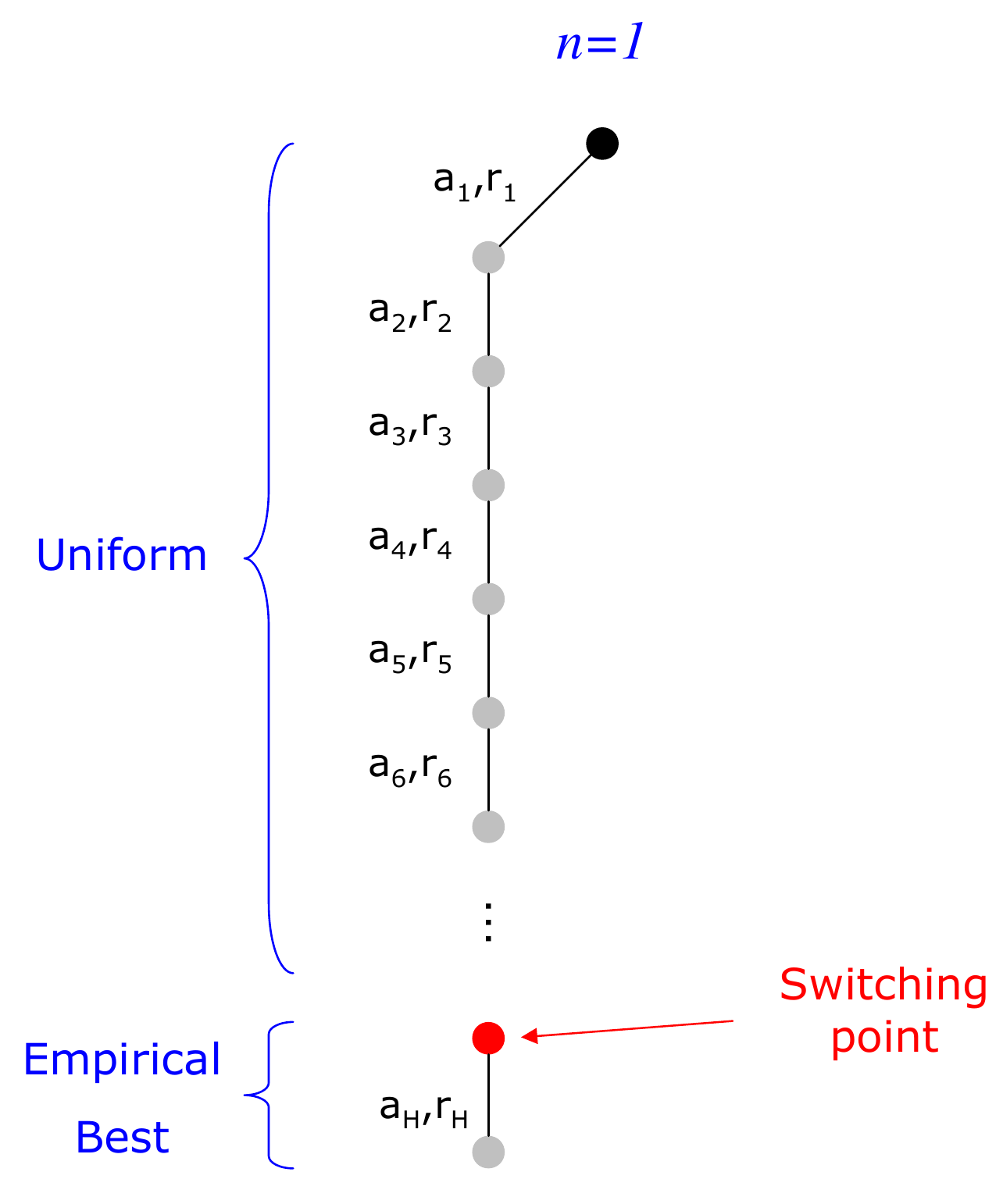}} &
\imagetop{\includegraphics[width=3cm]{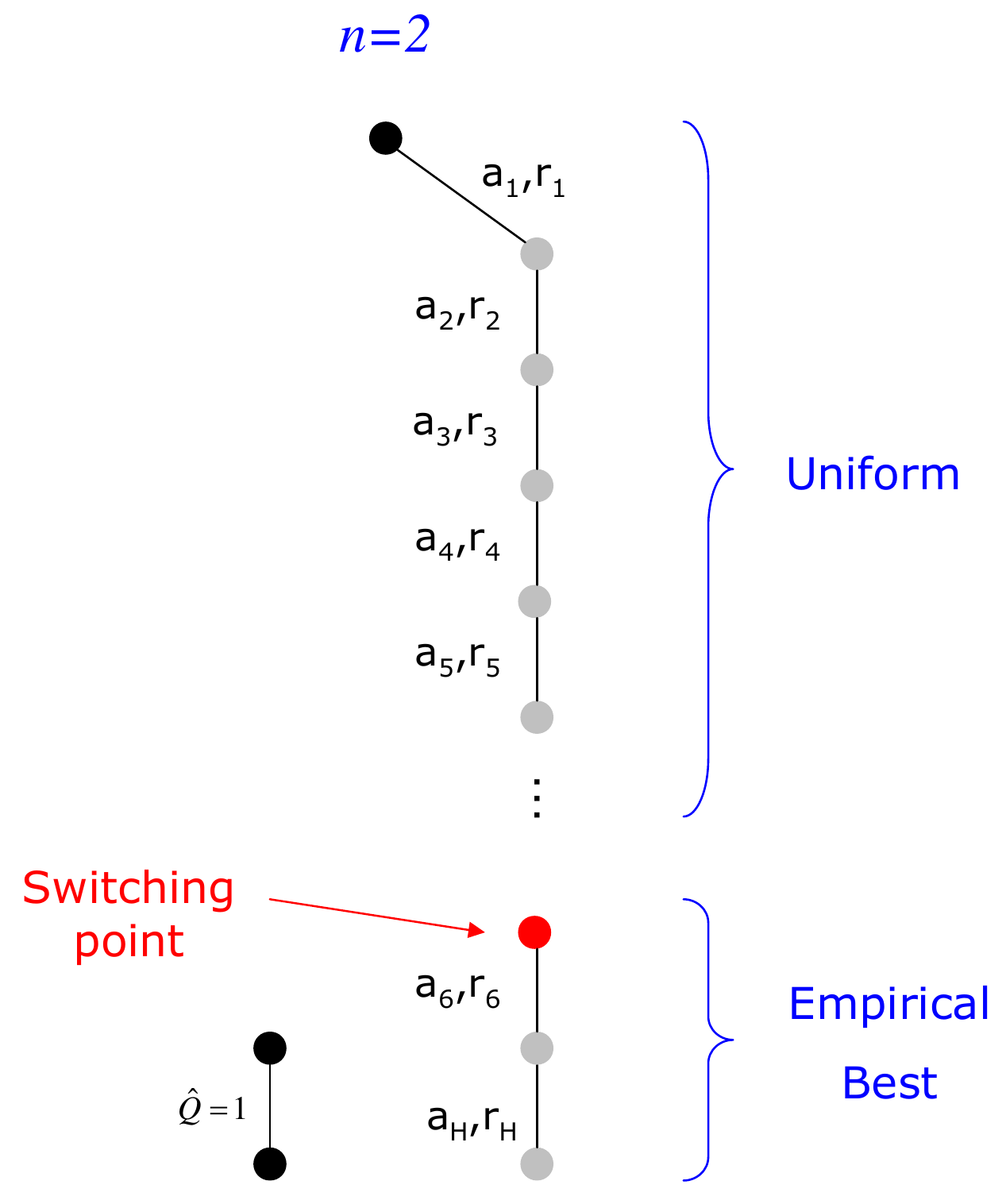}} &
\imagetop{\includegraphics[width=3cm]{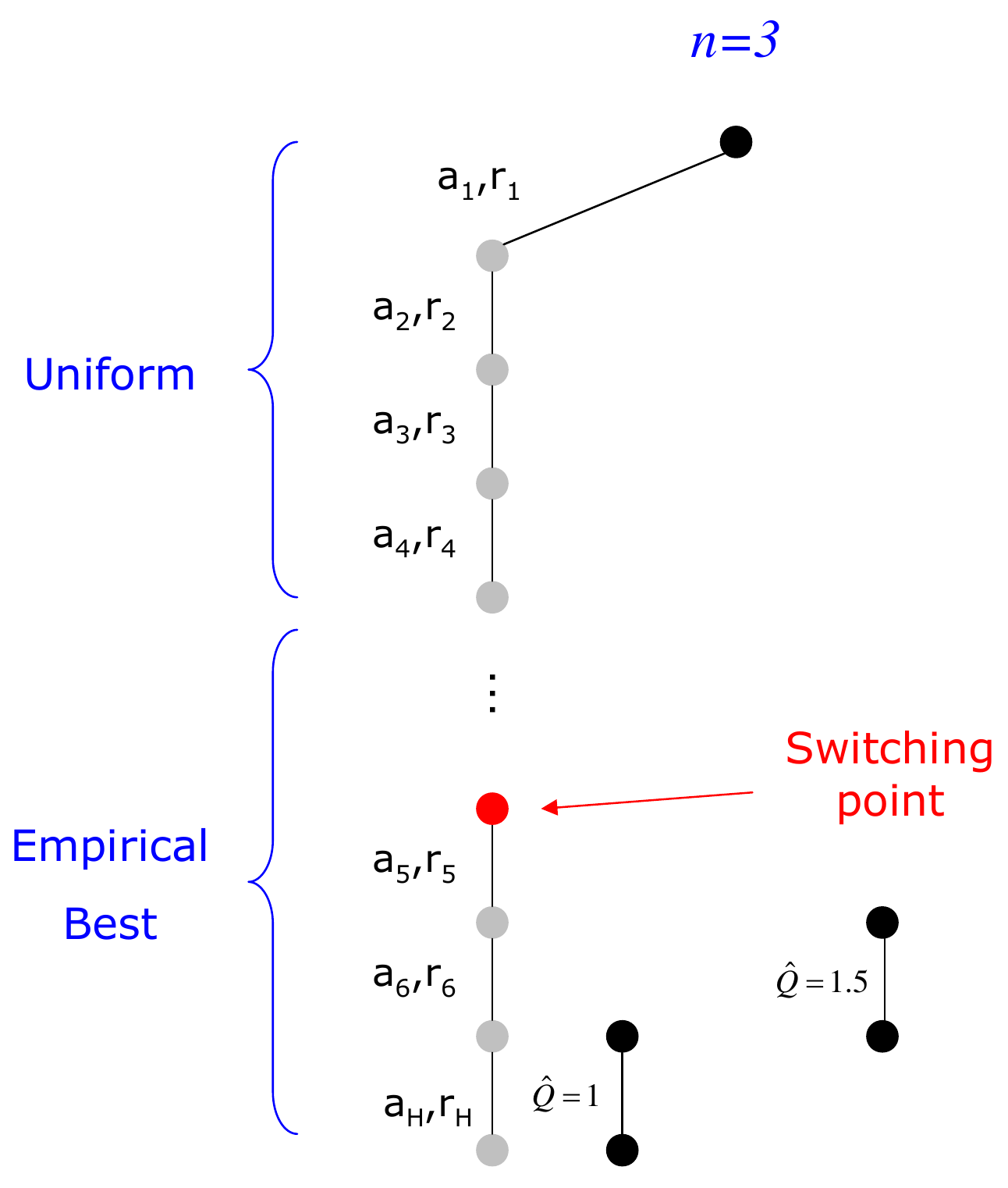}}
\\
 & & \\
\hline
\imagetop{\includegraphics[width=2cm]{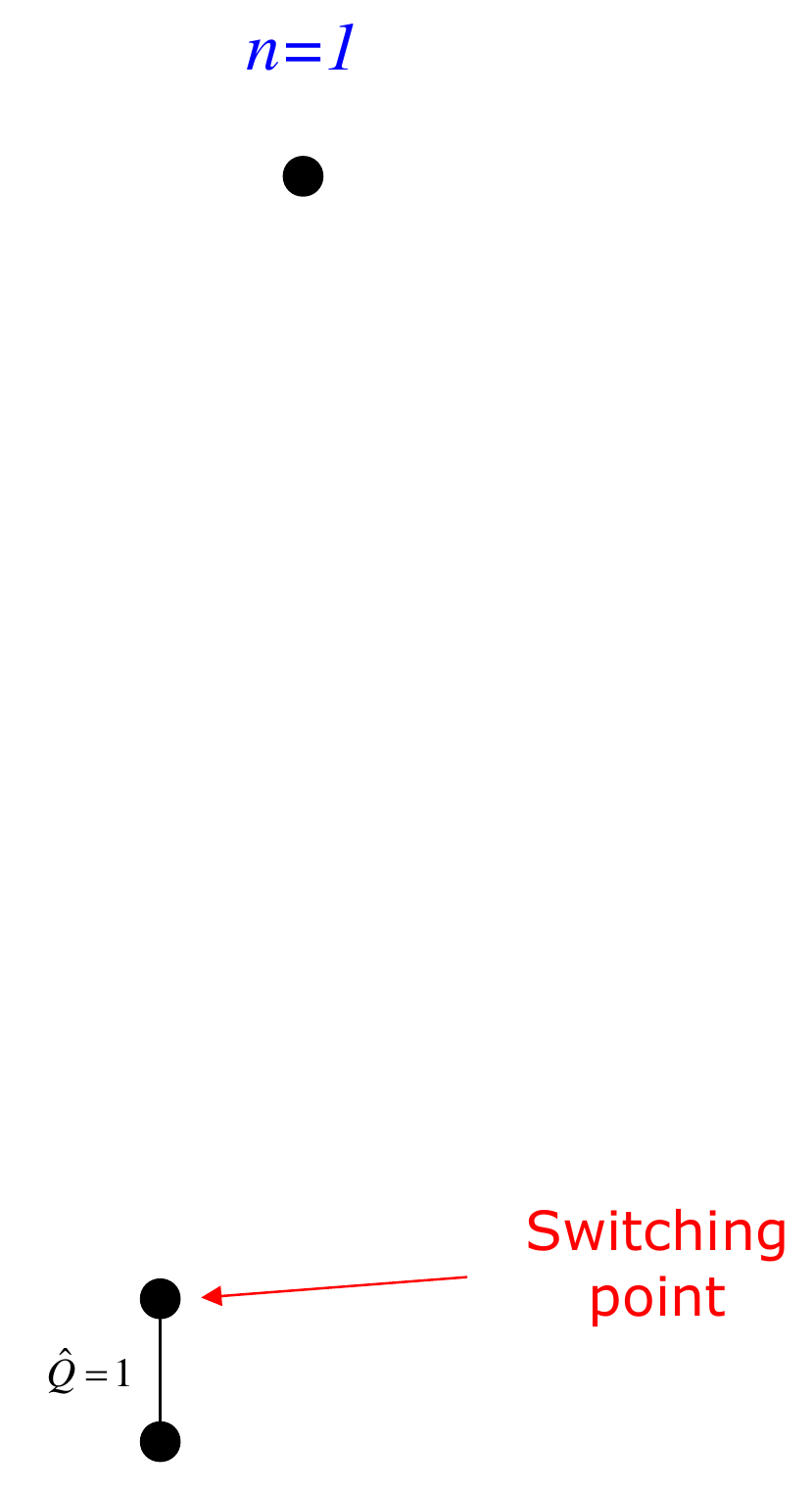}} &
\imagetop{\includegraphics[width=2cm]{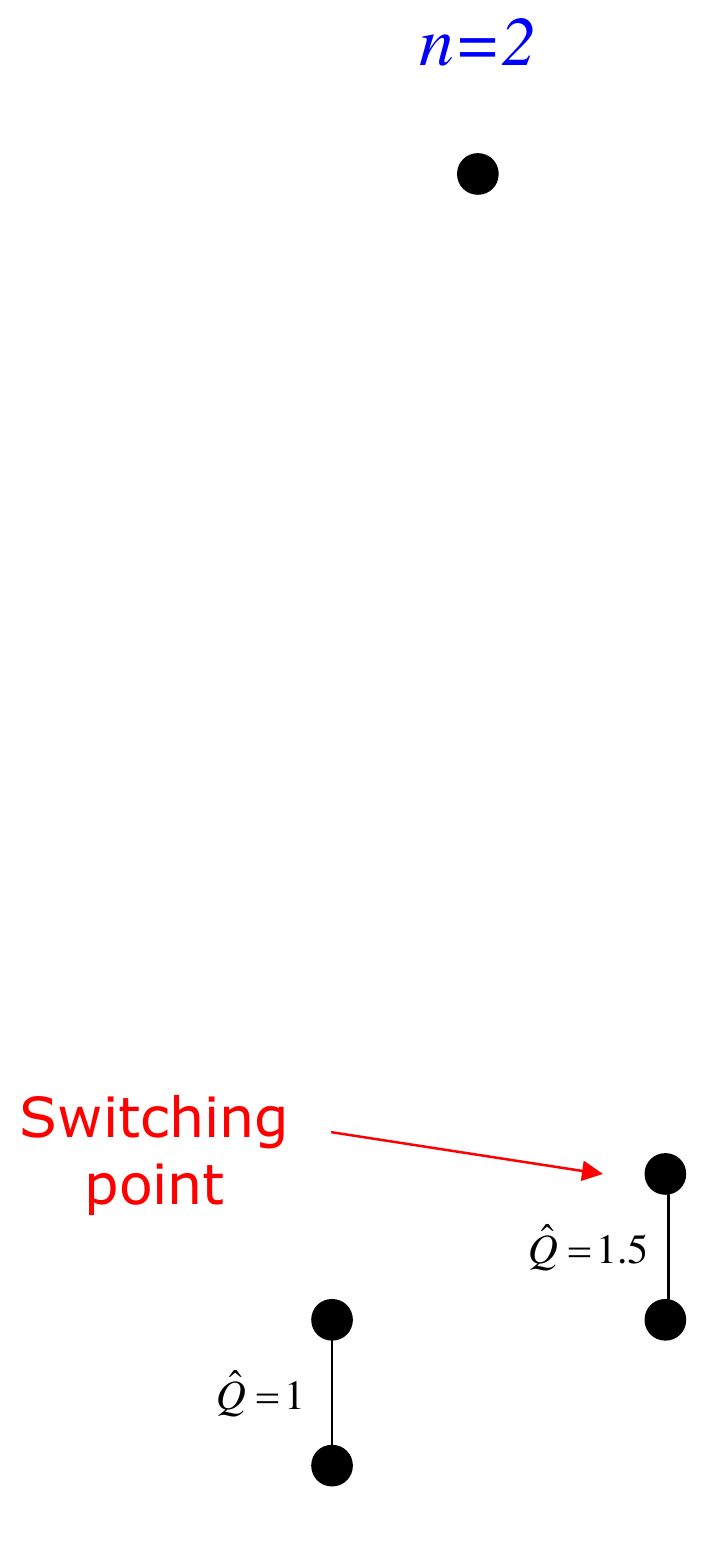}} &
\imagetop{\includegraphics[width=2.3cm]{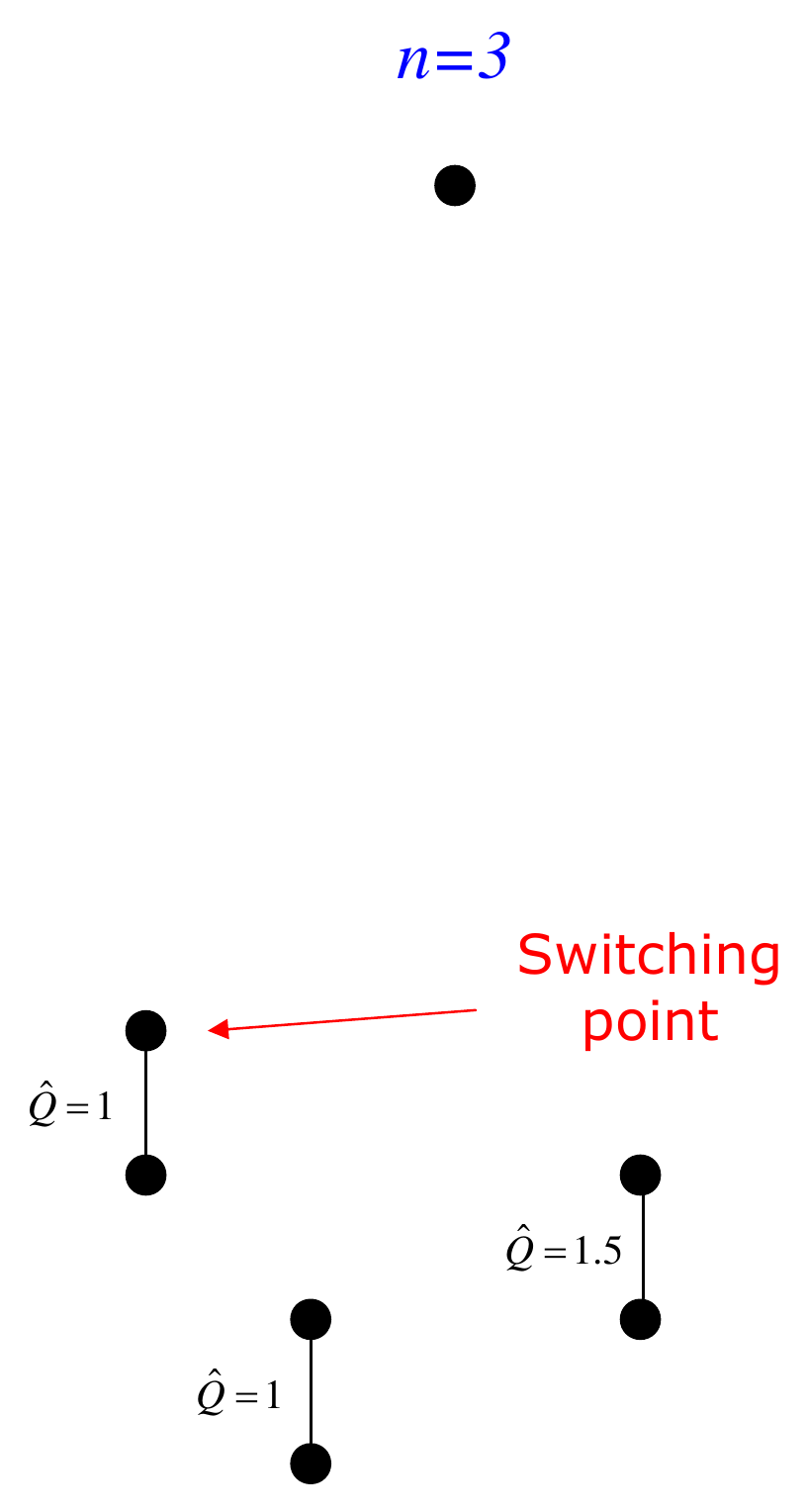}}
\\
 & & \\
\hline
\multicolumn{3}{c}{$\cdots$}\\
\hline
\imagetop{\includegraphics[width=2cm]{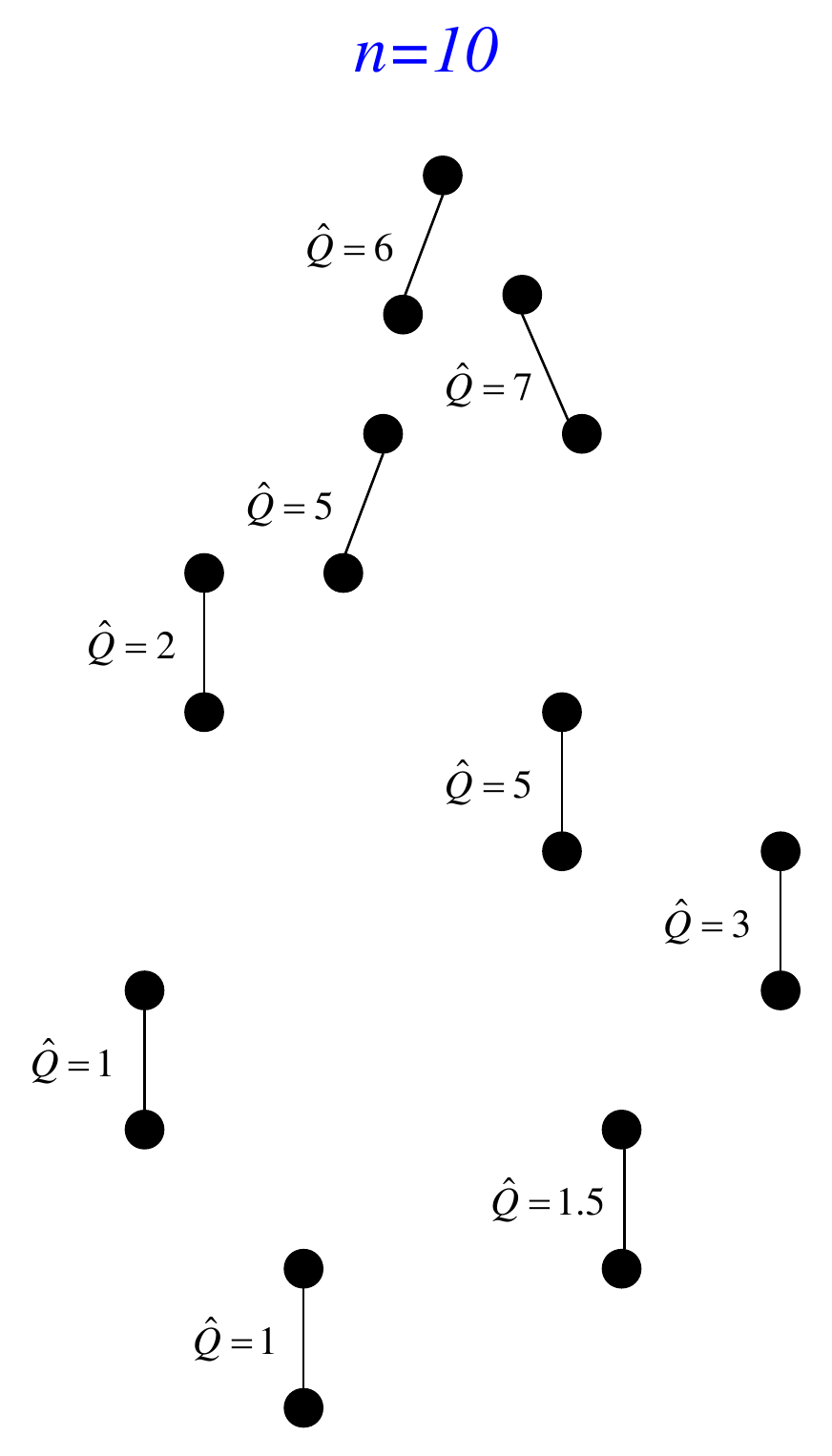}} &
\imagetop{\includegraphics[width=3.5cm]{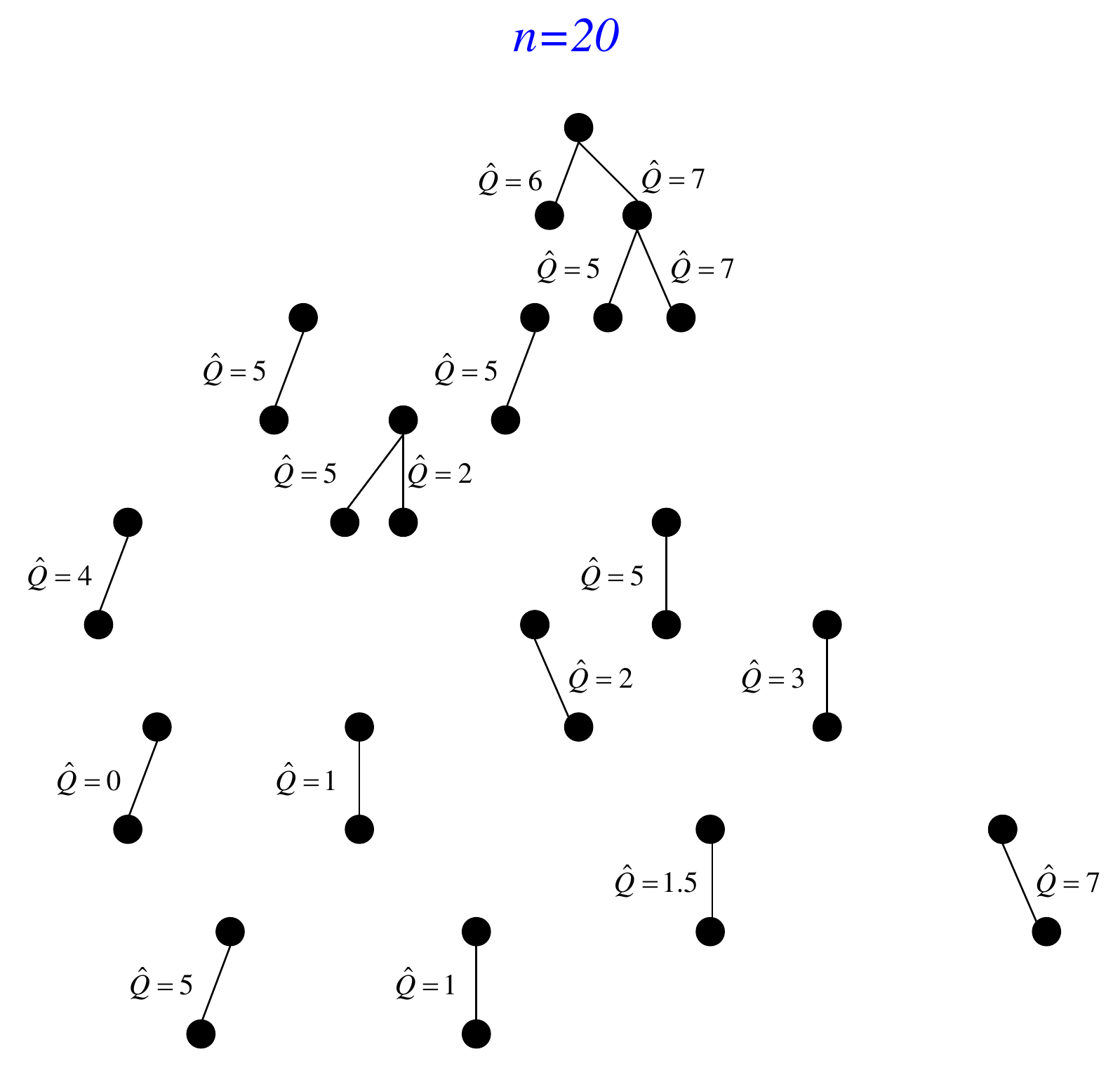}} &
\imagetop{\includegraphics[width=3.8cm]{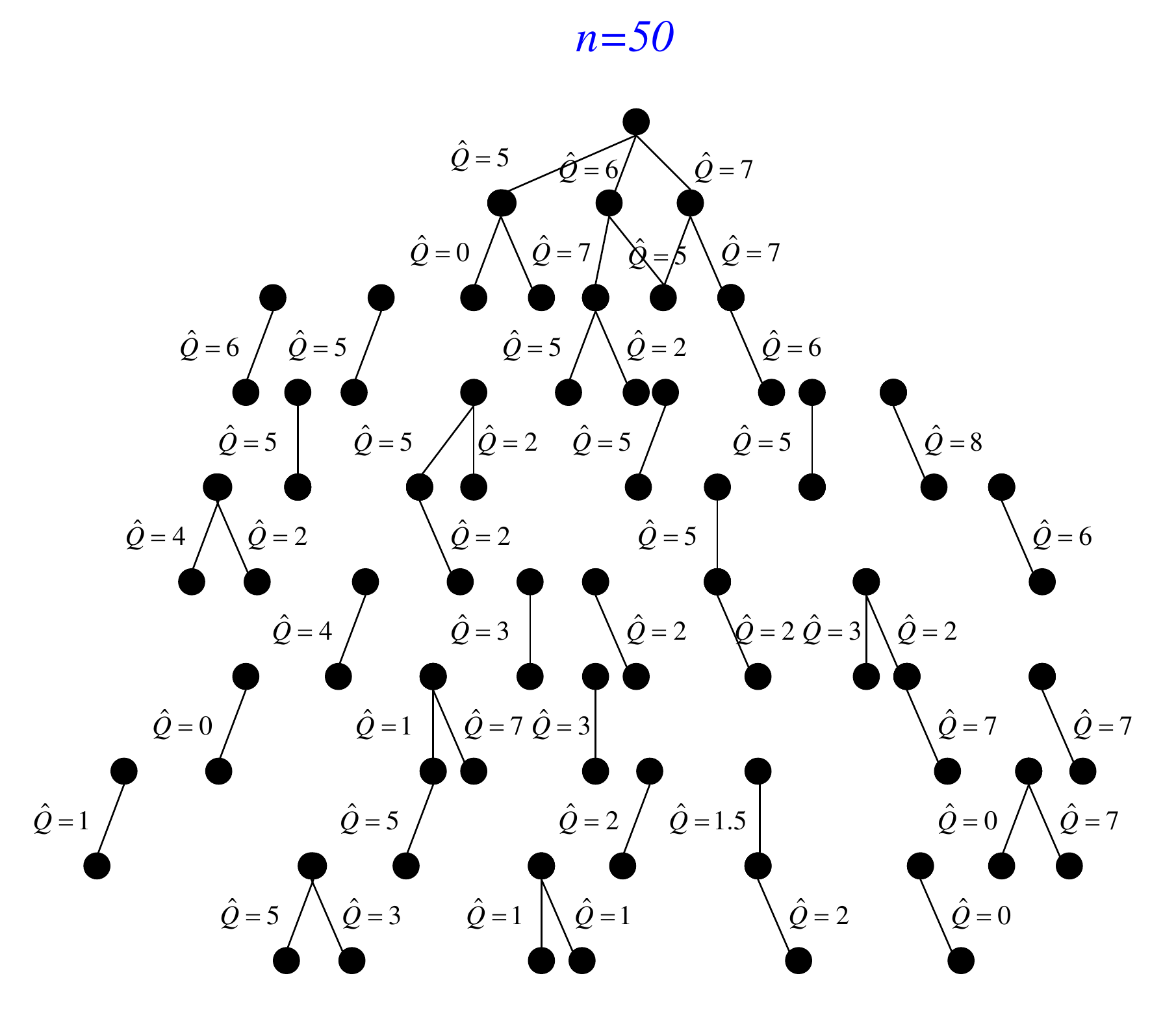}}\\
 & & \\
\hline
\end{tabular}
\end{center}
\caption{\label{fig:bruedynamics}Illustration of the $\BRUSH$ dynamics}
\end{figure}

Before we proceed with the formal analysis of $\BRUSH$, a few comments on it, as well as on the $\MCTREEnew$ sampling scheme in general, are in place. First, the template of $\MCTREEnew$ is rather general, and some of its parametrizations will not even guarantee convergence to the optimal action. This, for instance, will be the case with a (seemingly minor) modification of $\BRUSH$ to purely uniform estimation policy. In short,  $\MCTREEnew$ should be parametrized with care.
Second, while in what follows we focus on $\BRUSH$, other instances of $\MCTREEnew$ may appear to be empirically effective as well with respect to the reduction of simple regret over time. Some of them, similarly to $\BRUSH$, may also guarantee exponential-rate reduction of simple regret over time.
Hence, we clearly cannot, and do not, claim any uniqueness of $\BRUSH$ in that respect.
Finally, some other families of MCTS algorithms, more sophisticated that $\MCTREEnew$,  can  give rise to even more (formally and/or empirically) efficient optimizers of simple regret. The $\BRUSHa$ set of algorithms that we discuss later on is one such example.


\section{Upper Bounds on Simple Regret Reduction Rate with $\BRUSH$}

For the sake of simplicity, in our formal analysis of $\BRUSH$ we assume uniqueness of the optimal policy $\optpolicy$; that is, at each state $s$ and each number $\horizvar$ of steps-to-go, there is a single optimal action, and it is $\optpolicy(s,\horizvar)$. 
Let $\stree_{n}$ be the graph obtained by $\BRUSH$ after $n$ iterations, and let $\qcount_{\horizvar}(s,a)$ denote the accumulated value $\qcount(s,a)$ for $s$ at depth $\horiz-\horizvar$. 
For all state/steps-to-go pairs $(\state,\horizvar) \in \stree_{n}$, $\brushpolicy{n}(s,\horizvar)$ is a randomized strategy, uniformly choosing among actions $a$ maximizing $\qcount_{\horizvar}(s,a)$. 
We also use some additional auxiliary notation.
\begin{itemize}
\item[$K$] $= \max_{s\in S}{|A(s)|}$, i.e., the maximal number of actions per state. 
\item[$\minprob$] $= \min_{s,a,s':\transp(s,a,s')>0}\transp(s,a,s')$, i.e., the likelihood of the least likely (but still possible) outcome of an action in our problem.
\item [$\mindelta$] $=\min_{s,a}\sregret{1}{s}{a}$, i.e., the
smallest difference between the value of the optimal and a second-best action at a state with just one step-to-go.
\end{itemize}

Our key result on the $\BRUSH$ algorithm is Theorem~\ref{thm:brush_regret} below. The proof of Theorem~\ref{thm:brush_regret}, as well as of several required auxiliary claims, is given in Appendix~\ref{app:proofs}. Here we outline only the key issues  addressed by the proof, and provide a high-level flow of the proof in terms of a few central auxiliary claims.

\begin{theorem}
\label{thm:brush_regret}
Let $\BRUSH$ be called on a state $\initstate$ of an MDP $\mdp{\states,\actions,\transp,\reward}$ with rewards in $\left[0,1\right]$ and finite horizon $H$. There exist pairs of parameters $\paramgenone,\paramgentwo > 0$, dependent only on $\{p,d,K,H\}$, such that,
after $n > H$ iterations of $\BRUSH$, we have simple regret bounded as
  \begin{equation}\label{eq:brush_gen_reg}
    \expectation \regret{s,\reca{n}(\initstate,H)}{H}\leq 
    H \paramgenone\cdot e^{-\paramgentwo n},
\end{equation}
and choice-error probability bounded as
  \begin{equation}\label{eq:brush_gen_prob}
    \simpleprob \left\{\reca{n}(\initstate,H)\neq\optpolicy(\initstate,H)\right\}\leq 
    \paramgenone\cdot e^{-\paramgentwo n}.
  \end{equation}
In particular, these bounds hold for
\begin{equation}\label{eq:cgenone}
\paramgenone=\frac{4 K^{3H^{2}-2H}(H!)^3\prod_{h=1}^{H-1}(h!)^4 24^{H-1} 16^{(H-1)^2}}{d^{2H^2-4H+2}p^{3H^{2}-3H}},
\end{equation}
and
\begin{equation}\label{eq:cgentwo}
\paramgentwo=\frac{3d^{2H-2} p^{2H-1}}{2H 16^{H-1}(H!)^2K^{2H}}.
\end{equation}
\end{theorem}



Before we proceed any further, some discussion of the statements in Theorem~\ref{thm:brush_regret} are in place. First, the parameters $\paramgenone$ and $\paramgentwo$ in the bounds established by Theorem~\ref{thm:brush_regret} are problem-dependent: in addition to the dependance on the horizon $H$ and the choice branching factor $K$ (which is unavoidable), the parameters $\paramgenone$ and $\paramgentwo$ also depend on the distribution parameters $p$ and $d$. While it is possible that this dependence can be partly alleviated,~\citeA{bubeck:etal:tcs11} showed that distribution-free exponential bounds on the simple regret reduction rate cannot be achieved even in MABs, that is, even in single-step-to-go MDPs (see Remark~2 of~\citeA{bubeck:etal:tcs11}, which is based on a lower bound on the cumulative regret established by~\citeR{auer:etal:siam02}).
Second, the specific parameters $\paramgenone$ and $\paramgentwo$ provided by
Eqs.~\ref{eq:cgenone} and~\ref{eq:cgentwo} are worst-case for MDPs with  parameters $d$, $p$, and $K$, and the bound in Eq.~\ref{eq:brush_gen_reg} becomes effective after
$$n > \frac{\ln(\paramgenone)}{\paramgentwo} = O\left[ \left( \frac{KH}{pd} \right)^{\varepsilon H^2} \right]$$
iterations, for some small constant $\varepsilon > 1$. While there is still some gap with this transition period length and the transition period length of the theoretical $\craftyuniform$ algorithm (see Eq.~\ref{eq:transitioncrafty}), this gap is not that large.\footnote{Some of this gap can probably be eliminated by more accurate bounding in the numerous bounding steps towards the proof of Theorem~\ref{thm:brush_regret}. However, all such   improvements we tried made the already lengthy proof of Theorem~\ref{thm:brush_regret} even more involved.}

The proof of Lemma~\ref{thm:brush_qaccuracy} below constitutes the crux of the proof of Theorem~\ref{thm:brush_regret}. Once we have proven this lemma, the proof of Theorem~\ref{thm:brush_regret} stems from it in a more-or-less direct manner.

\begin{lemma}\label{thm:brush_qaccuracy}
Let $\BRUSH$ be called on a state $\initstate$ of an MDP $\mdp{\states,\actions,\transp,\reward}$ with rewards in $\left[0,1\right]$ and finite horizon $H$. For each $h\in\range{H}$, there exist parameters $c_h,c_h'>0$,  dependent only on $\{p,d,K,H\}$, such that,  for each state $\state$ reachable from $\initstate$ in $H-h$ steps and any $t > 0$, it holds that
\begin{equation}\label{eq:assmpt_eq}
\begin{split}
    \mathbb{P}\left\{\widehat{Q}_h\left(s,a\right)-Q_h\left(s,a\right)\geq {d\over 2}\;\middle|\;n_h\left(s,a\right)=t\right\} &\leq c_he^{-c_h't}, \\
    \mathbb{P}\left\{\widehat{Q}_h\left(s,a\right)-Q_h\left(s,a\right)\leq -{d\over 2}\;\middle|\;n_h\left(s,a\right)=t\right\} &\leq c_he^{-c_h't}.
\end{split}
\end{equation}
In particular, these bounds hold for 
\begin{equation}\label{eq:c_recursive_carmel1}
c_{h}=\frac{K^{2Hh+h^{2}-2H-1}(h!)^3\prod_{i=1}^{h-1}(i!)^4 24^{h-1} 16^{(h-1)^2}}{d^{2(h-1)^{2}}\cdot p^{2Hh+h^{2}-2H-h}},
\end{equation}
and
\begin{equation}\label{eq:c_recursive_carmel2}
c_{h}'=\frac{3d^{2(h-1)} p^{H+h-1}}{16^{h-1}(h!)^2 K^{H+h-1}}.
\end{equation}
\end{lemma}

The proof for Lemma~\ref{thm:brush_qaccuracy} is by induction on $h$. Starting with the induction basis for $h=1$, it is easy to verify that, by the Chernoff-Hoeffding inequality,
\begin{equation}\label{eq:h1_bound}
\mathbb{P}\left\{\left\vert\estq{s,a}{1}-Q_1\left(s,a\right)\right\vert\geq {d\over 2}\;\middle|\;\nsa{s,a}{1}=t\right\}\leq 2 e^{-{d^2\over 2}t},
\end{equation}
that is, the assertion is satisfied with $\paramone{1}=1$ and $\paramtwo{1}=\frac{d^2}{2}$. Now, assuming the claim holds for $h\geq 1$, below we outline the proof for $h+1$, relegating the actual proof in full detail to Appendix~\ref{app:proofs}.

In the proof for $h>1$, it is crucial to note the invalidity of applying the Chernoff-Hoeffding bound directly, as  was done in Eq.~\ref{eq:h1_bound}. There are two reasons for this. 
\begin{description}
\item[(F1)] For $h=1$, $\Estq$ is an {\em unbiased} estimator of $\Optq$, that is,
$\expectation\Estq = \Optq$. In contrast, the estimates inside the tree (at nodes with $h>1$) are {biased}. This bias stems from $\Estq$ possibly being based on numerous sub-optimal choices in the sub-tree rooted in $(s,h)$.
\item[(F2)] For $h=1$, the summands accumulated by $\Estq$ are independent. This is not so for $h>1$, where the accumulated reward depends on the selection of actions in subsequent nodes, which in turn depends on previous rewards.
\end{description}
However, we show that these deficiencies of $h>1$ can still be overcome through
a novel modification of the seminal Hoeffding-Azuma inequality. 

%

\begin{lemma}[Modified Hoeffding-Azuma inequality]\label{lemma:azuma_mod}
Let $\{X_i\}_{i=1}^\infty$ 
be a sequence of random variables with support $[0,h]$ and $\mu_i\triangleq\mathbb{E}X_i$. If $\lim_{i\rightarrow\infty}{\mu_{i}}=\mu$, and
%
\begin{equation}
\label{e:azuma1}
\mathbb{P}\left\{\mathbb{E}\left[X_i\;\middle|\;X_1,\ldots,X_{i-1}\right] \neq \mu\right\}\leq c_p e^{-c_ei},
\end{equation}
for some $0 < c_{p}$ and $0 < c_{e} \leq 1$, then, for all $0<\delta\leq\frac{h}{2}$, it holds that
\begin{eqnarray}
\label{eq:azumainternal1}
\mathbb{P}\left\{\sum_{i=1}^tX_i\geq\mu t+t\delta\right\}&\leq& \left[ 1 + c_p\frac{2h^2}{\delta^2 c_e^2} \right] \cdot e^{-\frac{3\delta^2 c_e}{2h^2}t},\\
\label{eq:azumainternal2}
\mathbb{P}\left\{\sum_{i=1}^tX_i\leq\mu t-t\delta\right\}&\leq& \left[ 1 + c_p\frac{2h^2}{\delta^2 c_e^2}\right] \cdot e^{-\frac{3\delta^2 c_e}{2h^2}t}.
\end{eqnarray}
\end{lemma}

Together with Lemma~\ref{thm:probe_bounds} below, the inequalities provided by Lemma~\ref{lemma:azuma_mod} allow us to prove the induction hypothesis in the proof of the central Lemma~\ref{thm:brush_qaccuracy}. 
Note that the specific bound in Lemma~\ref{lemma:azuma_mod} is selected so to  maximize  the exponent coefficient. For any $0\leq\beta\leq 1$, the probabilities of interest in Eqs.~\ref{eq:azumainternal1}-\ref{eq:azumainternal2} can also be bounded by
\[
\left[ 1+ \frac{c_p}{c_e\left(1-\beta\right)} e^{-\frac{c_e\left(1-\beta\right)}{2h^2}}\right]e^{-\frac{3\delta^2 c_e\beta}{2h^2}t };
\]
for further details, we refer the reader to Discussion~\ref{dsc:estimate_bound} in Appendix~\ref{app:proofs}.


\begin{definition}
\label{def:Xevent}
Let $\mdp{\states,\actions,\transp,\reward}$ be an MDP with rewards in $\left[0,1\right]$, planned for initial state $\initstate \in \states$ and finite horizon $H$. Let $\state$ be a state reachable from $\initstate$ with $h$ steps still to go, let $a$ be an action applicable in $s$, and let $\reca{t}$ be a policy induced by running $\BRUSH$ on $\initstate$ until exactly $t>0$ samples have finished their exploration phase with applying action $a$ at $s$ with $h-1$ steps still to go. Given that,
\begin{itemize}
\item $X_{t,h}(s,a)$ is a random variable, corresponding to the reward obtained by taking  $a$ at  $s$, and then following $\reca{t}$ for the remaining $h-1$ steps.
\item $E_{t,h}\left(s,a\right)$ is the event in which $X_{t,h}(s,a)$ is sampled along the optimal actions at each of the $h-1$ choice points delegated to $\reca{t}$.
\item $\delta_{t,h}\left(s,a\right)=Q_{h}\left(s,a\right)-\mathbb{E}\left[X_{t,h}(s,a)\right].$
\end{itemize}
\end{definition}

\begin{lemma}
\label{thm:probe_bounds}
Let $\mdp{\states,\actions,\transp,\reward}$ be an MDP with rewards in $\left[0,1\right]$, planned for initial state $\initstate \in \states$ and finite horizon $H$. Let $\state$ be a state reachable from $\initstate$ with $h+1$ steps still to go, and $a$ be an action applicable in $s$. Considering $E_{t,h+1}\left(s,a\right)$ and $\delta_{t,h+1}\left(s,a\right)$ as in Definition~\ref{def:Xevent}, for any $t > 0$,
if Lemma~\ref{thm:brush_qaccuracy} holds for horizon $h$, then
\begin{eqnarray}
\mathbb{P}\left\{\neg E_{t,h+1}\left(s,a\right)\right\}&\leq&2Kh\left(2+c_{h}\right)e^{-\frac{pc_{h}'}{6K}t},\\
\delta_{t,h+1}\left(s,a\right)&\leq& 2Kh^2\left(2+c_{h}\right)e^{-\frac{pc_{h}'}{6K}t}.
\end{eqnarray}
\end{lemma}



Together with a modified version of the Hoeffding-Azuma bound in Lemma~\ref{lemma:azuma_mod}, the bounds established in Lemma~\ref{thm:probe_bounds} allow us to derive concentration bounds for $\widehat{Q}_{h+1}$ around $Q_{h+1}$ as in Lemma~\ref{thm:induction_step} below, which serves the key building block for proving
the induction hypothesis in the proof of Lemma~\ref{thm:brush_qaccuracy}.

\begin{lemma}\label{thm:induction_step}
Let $\BRUSH$ be called on a state $\initstate$ of an MDP $\mdp{\states,\actions,\transp,\reward}$ with rewards in $\left[0,1\right]$ and finite horizon $H$. For each state $\state$ reachable $\initstate$ with $h+1$ steps still to go, each action $a$ applicable,  and any $t > 0$, it holds that
\begin{equation}
\mathbb{P}\left\{\left\vert\widehat{Q}_{h+1}\left(s,a\right)-Q_{h+1}\left(s,a\right) \right\vert\geq \frac{d}{2}\;\middle|\;n_{h+1}\left(s,a\right)=t\right\}
\leq \left(3456 \cdot \frac{K^3(h+1)^3 c_h}{d^2 p^2c_h'^2} \right) e^{-\frac{d^2 pc_{h}'}{16(h+1)^2K}}.
\end{equation}
\end{lemma}

\section{Learning With Forgetting and $\BRUSH(\alpha)$}

When we consider the evolution of action value estimates in $\BRUSH$  over time (as well as in all other Monte-Carlo algorithms for online MDP planning), we can see that, 
in internal nodes these estimates are based on biased samples that stem from the  selection of non-optimal actions at descendant nodes. This bias tends to shrink as more samples are accumulated down the tree. Consequently, the estimates become more accurate, the probability of selecting an optimal action increases accordingly, and the bias of ancestor nodes shrinks in turn. An interesting question in this context is: shouldn't we weigh differently samples obtained at different stages of the sampling process? 
Intuition tells us that biased samples still provide us with valuable information, especially when they are all we have, but the value of this information decreases as we obtain more and more accurate samples. Hence, in principle, putting more weight on samples with smaller bias could increase the accuracy of our estimates. The key question, of course, is which of all possible weighting schemes are both reasonable to employ and preserve the exponential-rate reduction of expected simple regret.

Here we describe $\BRUSH\left(\alpha\right)$, an algorithm that generalizes $\BRUSH \equiv \BRUSH(1)$ by basing the estimates only on the $\alpha$ fraction of most recent samples. 
We discuss the value of this addition both from the perspective of the formal guarantees, as well as from the perspective of empirical prospects.
$\BRUSH(\alpha)$ differs from $\BRUSH$ in two points:

\begin{itemize}
\item In addition to the variables $\pcount(\state,\action)$ and $\qcount(\state,\action)$, each node/action pair $(\state,\action)$ in $\BRUSH(\alpha)$ is associated with a {\em list} $\slist(\state,\action)$ of rewards, collected at each of the $\pcount(\state,\action)$ samples that are responsible for the current estimate $\qcount(\state,\action)$.
\item When a sample $\probe = \tuple{s_{0},a_{1},s_{1},\dots,a_{k},s_{k}}$ is issued at iteration $n$, and  $\updatestat$ updates the variables at $x = (s_{\spfunc(n)-1},a_{\spfunc(n)})$, that update is done not according to Eq.~\ref{e:uctupdate} as in $\BRUSH$, but according to:
\begin{equation}
\label{e:brushalphaupdate}
\begin{split}
\pcount(x) & \leftarrow\;\;  \pcount(x) + 1,\\
\slist(x) [ \pcount(x) ] & \leftarrow\;\; \sum_{i=\spfunc(n)-1}^{k-1}\reward(s_{i},a_{i+1},s_{i+1}),\\
\qcount(x) & \leftarrow \;\;
\frac{1}{\lceil \alpha \cdot\pcount(x) \rceil} \sum_{i= \pcount(x)-\lceil \alpha \cdot\pcount(x)\rceil}^{\pcount(x)}{  \slist(x)[i]  }.
\end{split}
\end{equation}
\end{itemize}

\begin{theorem}\label{thm:brue_alpha}
Let $\BRUSH\left(\alpha\right)$ be called on a state $\initstate$ of an MDP $\mdp{\states,\actions,\transp,\reward}$ with rewards in $\left[0,1\right]$ and finite horizon $H$. There exist pairs of parameters $\paramgenone,\paramgentwo > 0$, dependent only on $\{\alpha, p,d,K,H\}$, such that,
after $n > H$ iterations of $\BRUSH$, we have simple regret bounded as
  \begin{equation}\label{eq:brushalpha_gen_reg}
    \expectation \regret{s,\reca{n}(\initstate,H)}{H}\leq 
    H \paramgenone\cdot e^{-\paramgentwo n},
\end{equation}
and choice-error probability bounded as
  \begin{equation}\label{eq:brushalpha_gen_prob}
    \simpleprob \left\{\reca{n}(\initstate,H)\neq\optpolicy(\initstate,H)\right\}\leq 
    \paramgenone\cdot e^{-\paramgentwo n}.
  \end{equation}
\end{theorem}

The proof for Theorem~\ref{thm:brue_alpha}  follows from Lemma~\ref{thm:brushalpha_qaccuracy} below similarly to the way Theorem~\ref{thm:brush_regret} follows from Lemma~\ref{thm:brush_qaccuracy}.
Note that in Theorem~\ref{thm:brue_alpha}
we do not provide explicit expressions for the constants $c$ and $c'$ as we did in Theorem~\ref{thm:brush_regret} (for $\alpha=1$). This is because the expressions that can be extracted from the recursive formulas in this case do not bring much insight. However, we discuss the potential benefits of choosing $\alpha<1$ in the context of our proof of Theorem~\ref{thm:brue_alpha}.

\begin{lemma}\label{thm:brushalpha_qaccuracy}
Let $\BRUSH(\alpha)$ be called on a state $\initstate$ of an MDP $\mdp{\states,\actions,\transp,\reward}$ with rewards in $\left[0,1\right]$ and finite horizon $H$. For each $h\in\range{H}$, 
there exist parameters $c_h,c_h'>0$,  dependent only on $\{\alpha, p,d,K,H\}$, such that,  for each state $\state$ reachable from $\initstate$ in $H-h$ steps and any $t > 0$, it holds that
\begin{equation}\label{eq:assmpt_eq_new}
\begin{split}
    \mathbb{P}\left\{\widehat{Q}_h\left(s,a\right)-Q_h\left(s,a\right)\geq {d\over 2}\;\middle|\;n_h\left(s,a\right)=t\right\} &\leq c_he^{-c_h't}, \\
    \mathbb{P}\left\{\widehat{Q}_h\left(s,a\right)-Q_h\left(s,a\right)\leq -{d\over 2}\;\middle|\;n_h\left(s,a\right)=t\right\} &\leq c_he^{-c_h't}.
\end{split}
\end{equation}
\end{lemma}
The proof for Lemma~\ref{thm:brushalpha_qaccuracy} is by induction, following the same line of the proof for Lemma~\ref{thm:brush_qaccuracy}. In fact, it deviates  from the latter only in
the application of the modified Hoeffding-Azuma inequality, which has to be further modified to capture the partial sums as in $\BRUSH(\alpha)$.

\begin{lemma}[Modified Hoeffding-Azuma inequality for partial sums]\label{lemma:azuma_mod_alpha}
Let $\{X_i\}_{i=1}^\infty$ be a sequence of random variables with support $[0,h]$ and $\mu_i\triangleq\mathbb{E}X_i$. If $\lim_{i\rightarrow\infty}{\mu_{i}}=\mu$, and
%
\begin{equation}
\label{eq:azuma1_alpha}
\mathbb{P}\left\{\mathbb{E}\left[X_i\;\middle|\;X_1,\ldots,X_{i-1}\right] \neq \mu\right\}\leq c_p e^{-c_ei},
\end{equation}
for some $0 < c_{p}$ and $0 < c_{e} \leq 1$, then, for all $0<\delta\leq\frac{h}{2}$, it holds that
\begin{eqnarray}
\mathbb{P}\left\{\sum_{i=t-\lceil\alpha t\rceil}^tX_i\geq\mu t+t\delta\right\}&\leq&  \left[1 +
\frac{c_p }{c_e(1-\alpha)} e^{-c_e(1-\alpha)^2t}\right]e^{-\frac{3\delta^2 c_e}{2h^2}\alpha t},\\
\mathbb{P}\left\{\sum_{i=t-\lceil\alpha t\rceil}^tX_i\leq\mu t-t\delta\right\}&\leq& \left[1 +
\frac{c_p }{c_e(1-\alpha)} e^{-c_e(1-\alpha)^2t}\right]e^{-\frac{3\delta^2 c_e}{2h^2}\alpha t}.
\end{eqnarray}
\end{lemma}


Considering the benefits of ``sample forgetting" as in $\BRUSH(\alpha)$,
let us compare the bound in Lemma~\ref{lemma:azuma_mod_alpha} to the bound
\[e^{-\frac{3\delta^2 c_e\beta}{2h^2}t }\left[ 1+ \frac{c_p}{c_e\left(1-\beta\right)} e^{-\frac{c_e\left(1-\beta\right)}{2h^2}}\right],\]
provided by Lemma~\ref{lemma:azuma_mod} for $\BRUSH$, that is, when {\em all} accumulated  samples are averaged. While both bounds are very similar, the exponent of the second exponential term  is multiplied  for $\BRUSH(\alpha<1)$ by $\left(1-\alpha\right) t$.
This poses a tradeoff: Decreasing $\alpha$ reduces the sampling bias, and thus decreases  the term $\frac{c_p}{c_e}$, but increases the other exponential term with no leading constant.  Obviously, since there is no bias at leaf nodes, it makes no sense to set $\alpha<1$ there. However, as we go further up the tree, the bias tends to grow ($\frac{c_p}{c_e}>>1$), but we also expect to have more samples ($t$ is larger). Thus, from the perspective of formal guarantees, it seems appealing to choose  smaller values of $\alpha$. Nevertheless, we do not try to optimize here the value of $\alpha$: First,  optimizing bounds doesn't necessarily lead to optimized  empirical accuracy. Second, the underlying optimization would have to be specific to each horizon $h$ and each sample  size $t$ (which is obviously out of the question), and thus anyway we would have to consider only some rough approximations to this optimization problem.
Finally,  biased samples in practice might be more valuable than what the theory suggests, as long as all actions at the same state/steps-to-go decision point experience a similar bias.


\section{Experimental Evaluation}
We have evaluated $\BRUSH$ empirically on the MDP sailing domain~\cite{PeretG:ecai04} that was used in previous works for evaluating MC planning algorithms~\cite{PeretG:ecai04,uct,tolpin:shimony:aaai12}, as well as on random game trees used in the original  empirical evaluation of $\UCT$~\cite{uct}.

In the sailing domain, a sailboat navigates to a destination on an 8-connected grid representing a marine environment, under fluctuating wind conditions. The goal is to reach the destination as quickly as possible, by choosing at each grid location a neighbor location to move to. The duration of each such move depends on the direction of the move ({\em ceteris paribus}, diagonal moves take $\sqrt{2}$ more time than straight moves), the direction of the wind relative to the sailing direction (the sailboat cannot sail against the wind and moves fastest with a tail wind), and the tack. The direction of the wind changes over time, but its strength is assumed to be fixed. This sailing problem can be formulated as a goal-driven MDP over finite state space and a finite set of actions, with each state capturing the position of the sailboat, wind direction, and tack.

In a goal-driven MDP, the lengths of the paths to a terminal state are not necessarily bounded, and thus it is not entirely clear to what depth $\BRUSH$ shall construct its  tree.
In the sailing domain, we chose $H$ to be $4\times n$, where $n$ is the grid-size of the  problem instance, as it is unlikely that the optimal path between any two locations on the grid will be larger than a complete encircling of the considered area. We note, however, that the recommendation-oriented samples $\bar{\rho}$ always end at a terminal state, similar to the rollouts issued by $\UCT$ and $\GCT$.

\begin{figure*}[t]
\begin{center}
\begin{tabular}{cc}
\begin{minipage}{0.47\textwidth}
\includegraphics[width=\textwidth]{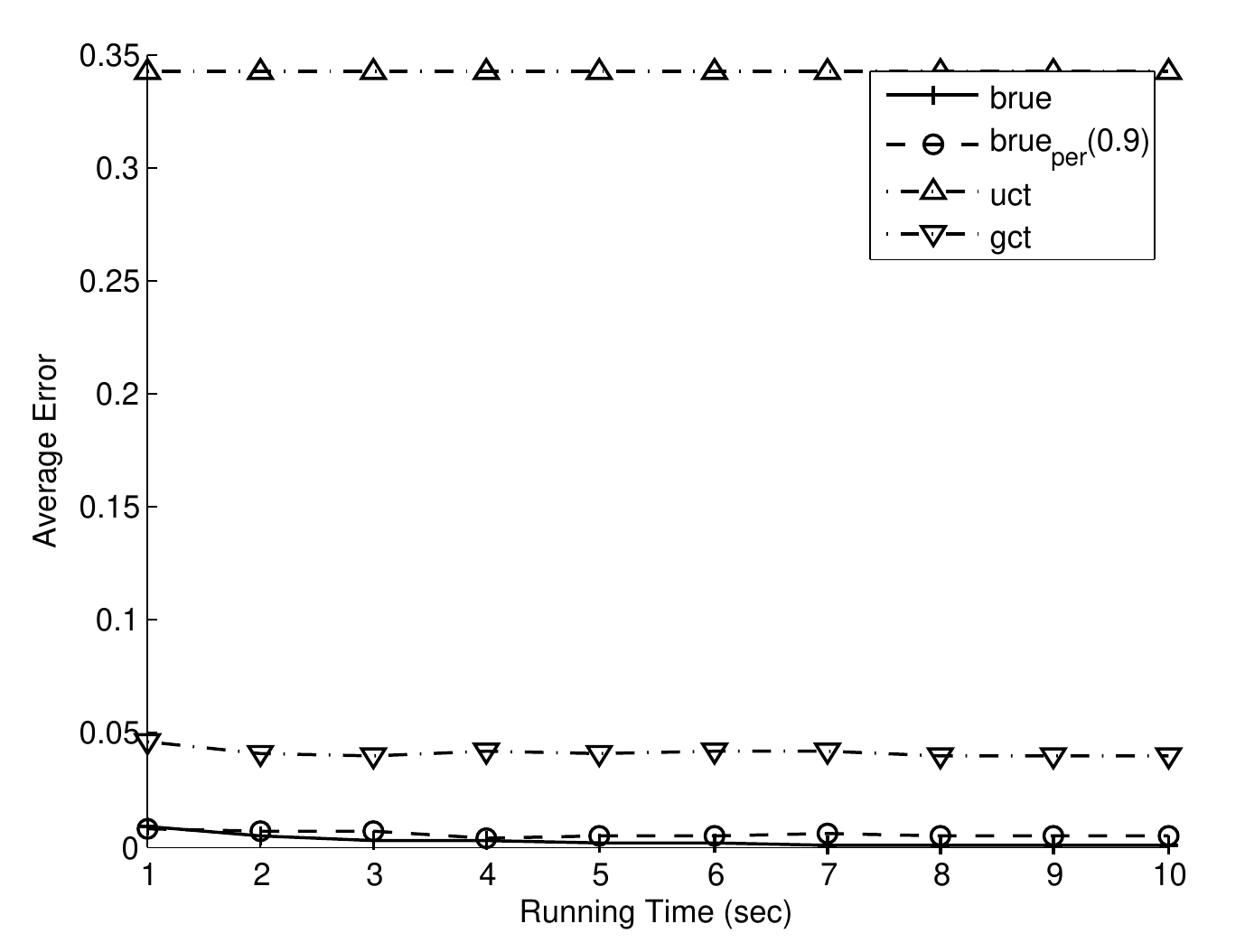}
\end{minipage}
&
\begin{minipage}{0.47\textwidth}
\includegraphics[width=\textwidth]{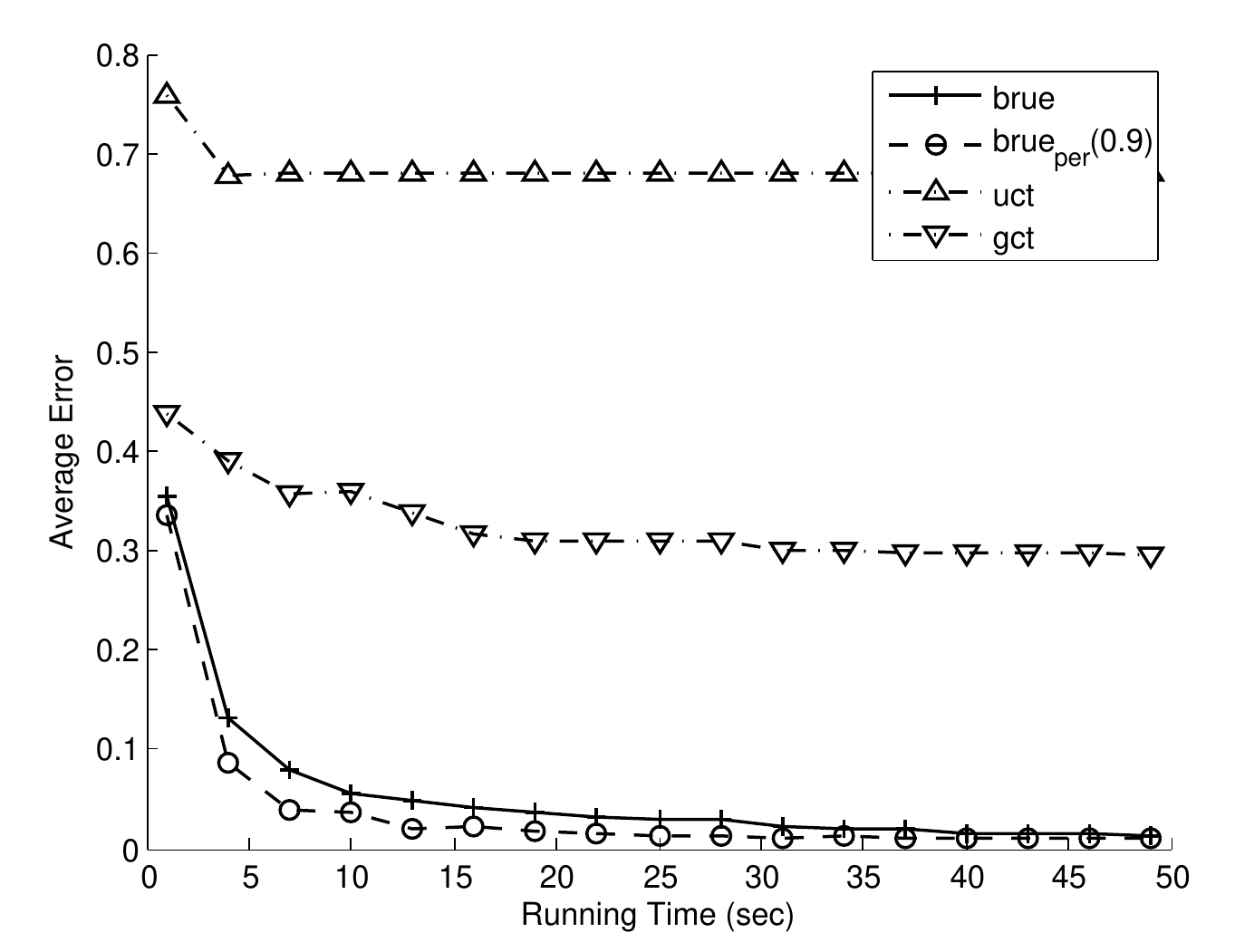}
\end{minipage}\\
{ $5\times 5$} & { $10\times 10$}\\
\begin{minipage}{0.47\textwidth}
\includegraphics[width=\textwidth]{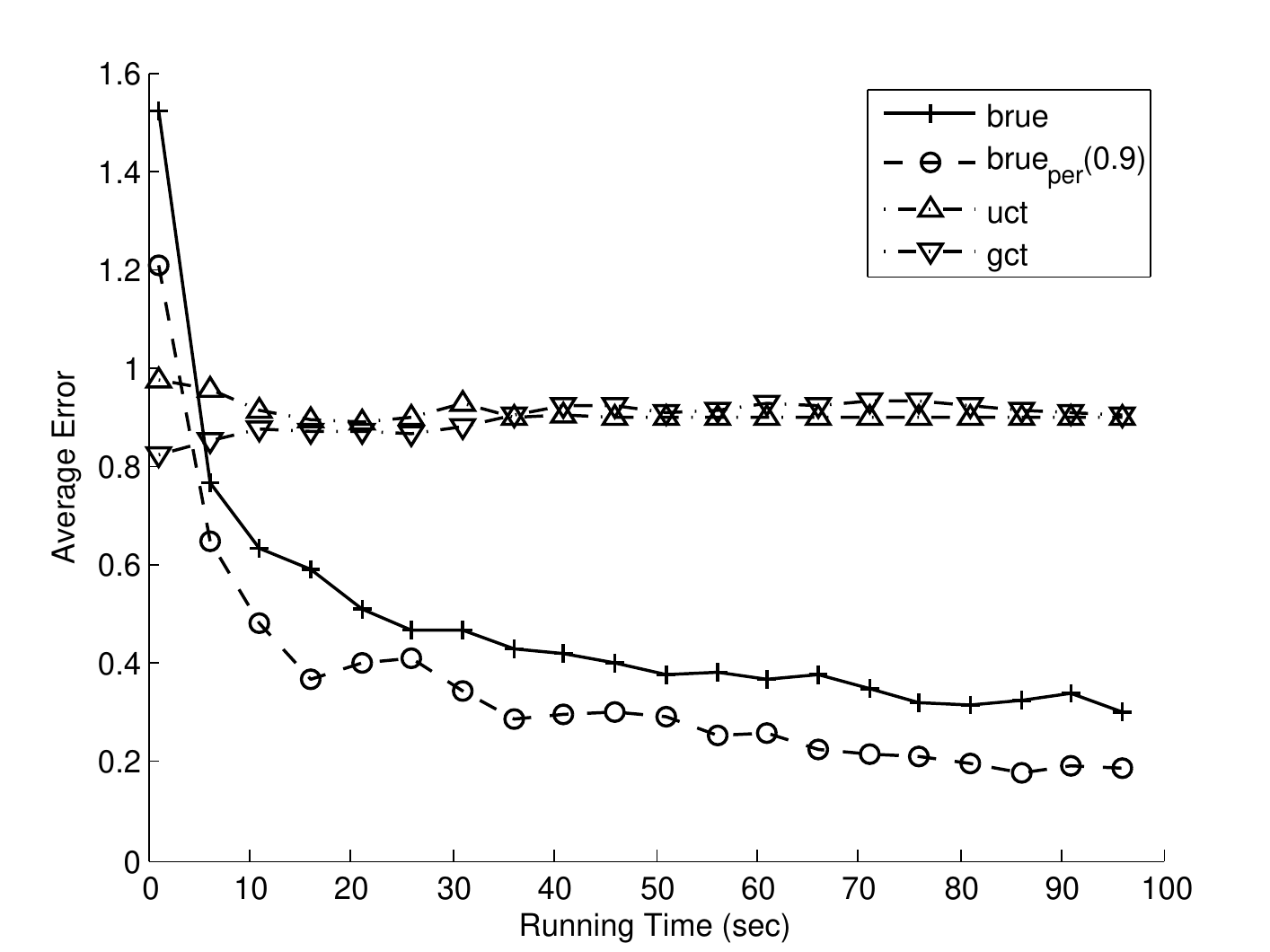}
\end{minipage}
&
\begin{minipage}{0.47\textwidth}
\includegraphics[width=\textwidth]{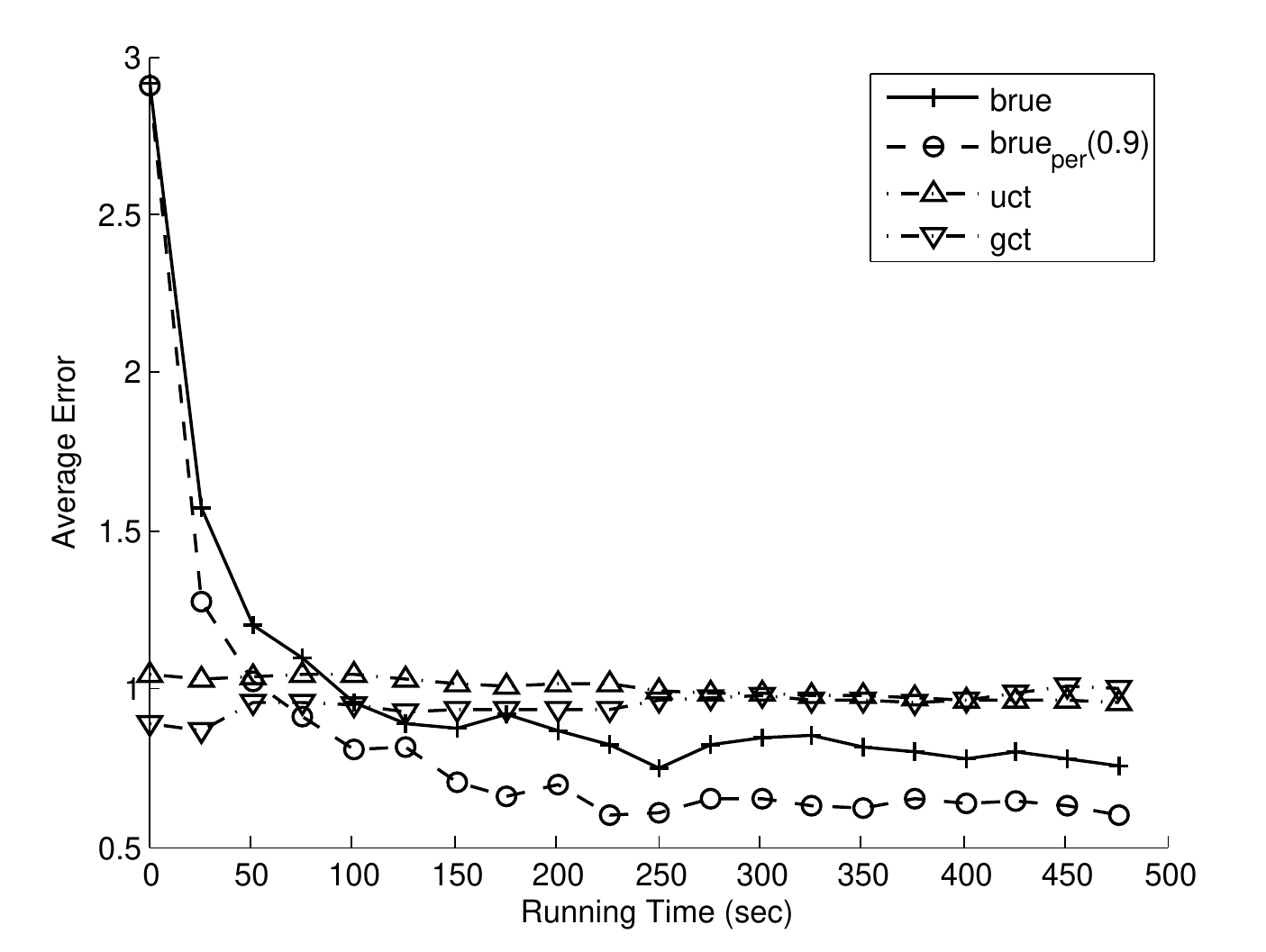}
\end{minipage}\\
{ $20\times 20$} &
{ $40\times 40$}
\end{tabular}
\end{center}
\caption{\label{fig:sailing} Empirical performance of $\BRUSH$, $\BRUSH(0.9)$, $\UCT$, and $\GCT$ (denoted as {\sf GCT}, for short) in terms of the average error on sailing domain problems on $n\times n$ grids with $n \in \{5,10,20,40\}$.}
\end{figure*}

\begin{figure*}[t]
\begin{center}
\begin{tabular}{cc}
\begin{minipage}{0.47\textwidth}
\includegraphics[width=\textwidth]{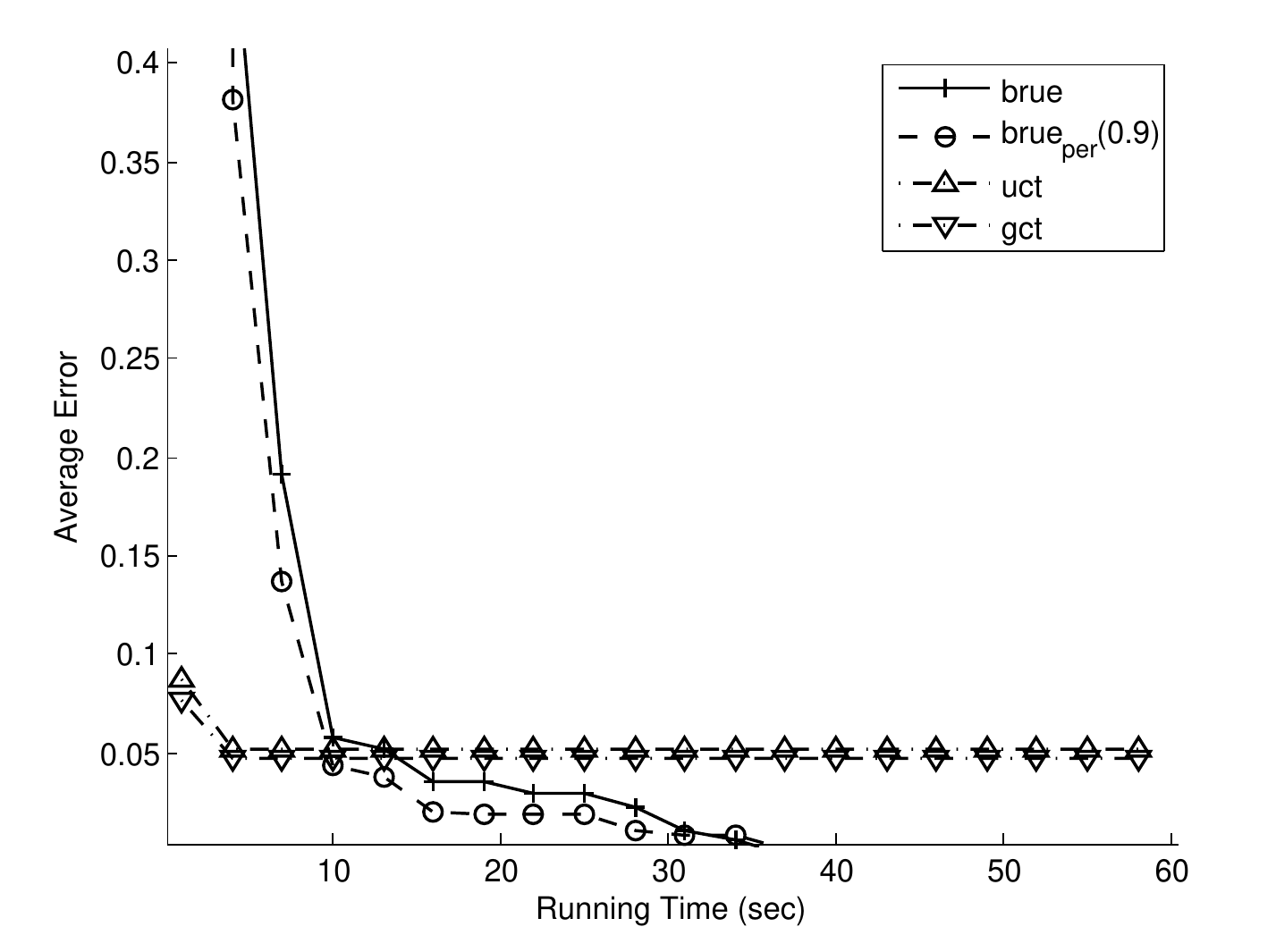}
\end{minipage}
&
\begin{minipage}{0.47\textwidth}
\includegraphics[width=\textwidth]{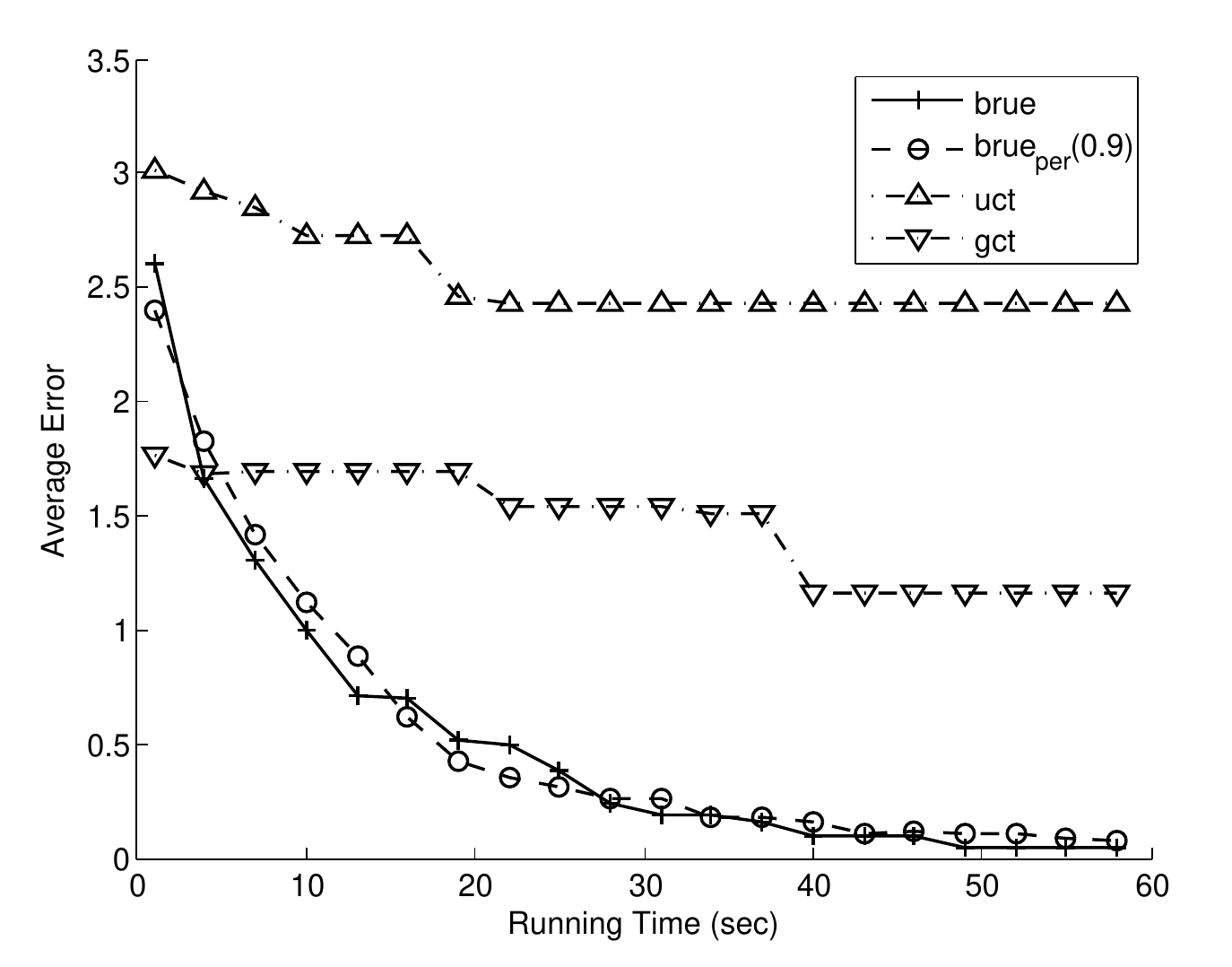}
\end{minipage}\\
 { $B=6/D=6$} &  { $B=2/D=16$}
\end{tabular}
\end{center}
\caption{\label{fig:pgame} Empirical performance of $\BRUSH$, $\UCT$, and $\GCT$ (denoted as {\sf GCT}) in terms of the average error on the random game trees with branching factor $B$ and tree depth $D$.}
\end{figure*}

Snapshots of the results for different grid sizes 
are shown 
in Figure~\ref{fig:sailing}.
We compared $\BRUSH$ with two $\MCTREE$-based algorithms: the $\UCT$ algorithm, and a recent modification of $\UCT$, $\GCT$, obtained from the former by replacing the $\ucbone$ policy {\em at the root node} with the $\epsilon$-greedy policy~\cite{tolpin:shimony:aaai12}. The motivation behind the design of $\GCT$ was to improve the  empirical simple regret of $\UCT$, and the results for $\GCT$ reported by \cite{tolpin:shimony:aaai12} (and confirmed by our experiments here) are very impressive.
We also show the results for $\BRUSH_{\mbox{\scriptsize per}}(0.9)$, a slight modification of $\BRUSH(0.9)$ with a more permissive update scheme: Instead of updating only the state-action node at the level of the switching point, we also update any ancestor for which either not all applicable actions have been sampled or the chosen action was identical to the best empirical one.

All four algorithms were implemented within a single software infrastructure.
As suggested by more recent works on $\UCT$, the exploration coefficient for $\UCT$ and $\GCT$ (parameter $c$ in Eq.~\ref{e:uctselect}) was set to the empirical best value of an action at the decision point~\cite{keller:eyerich:icaps12}. (This setting of the exploration coefficient resulted in better performance of both $\UCT$ and $\GCT$ than with the settings reported on the sailing domain in the respective original publications.) The $\epsilon$ parameter in $\GCT$ was set to $0.5$ as in the experiments of~\citeR{tolpin:shimony:aaai12}.
Each algorithm was run on 1000 randomly chosen initial states $\initstate$, and the performance of the algorithm was assessed in terms of the average error $Q(\initstate,a)-V(\initstate)$, that is, the difference between the true values of the action $a$ chosen by the algorithm and that of the optimal action $\optpolicy(\initstate)$.
Consistently with the results reported by Tolpin and Shimony~\citeyear{tolpin:shimony:aaai12}, on the smaller tasks $\GCT$ outperformed $\UCT$ by a very large margin, with the latter exhibiting very little improvement over time even on the smallest, $5\times 5$, grids.
The difference between $\GCT$ and $\UCT$ on the larger tasks was less notable.
In turn, $\BRUSH$ substantially outperformed $\GCT$, with the improvement being consistent except for relatively short planning deadlines, and $\BRUSH_{\mbox{\scriptsize per}}(0.9)$ performed even better than $\BRUSH$.

The above allows us to conclude that $\BRUSH$ is not only attractive in  terms of  the formal performance guarantees, but can also be very effective in practice for online planning.  Likewise, the  ``learning with forgetting'' extension of $\BRUSH(\alpha)$ also has its practical merits.
Under the same parameter setting of $\UCT$ and $\GCT$, we have also evaluated the three algorithms in a domain of random game trees whose goal is a simple modeling of two-person zero-sum games such as Go, Amazons and Globber. In such games, the winner is decided by a global evaluation of the end board, with the evaluation employing this or another feature counting procedure; the rewards thus are associated only with the terminal states.
The rewards are calculated by first assigning values to moves, and then summing up these  values along the paths to the terminal states. Note that the move values are used for the tree construction only and are not made available to the players. The values are chosen uniformly from $\left[0,127\right]$ for the moves of MAX, and from $\left[-127,0\right]$ for the moves of MIN. The players act so to (depending on the role)
maximize/minimize their individual payoff: the aim of MAX is to reach terminal $s$ with as high $R(s)$ as possible, and the objective of MIN is similar, {\em mutatis mutandis}.
%
This simple game tree model is similar in spirit to many other game tree models used in previous work~\cite{uct,pgame}, except that the success/failure of the players in measured not on a ternary scale of win/lose/draw, but via the actual payoffs they receive.
We  ran some experiments with two different settings of the branching factor $(B)$ and tree depths $(D)$. As in the sailing domain, we compared the convergence rate obtained by $\BRUSH$, $\UCT$ and $\GCT$. Figure~\ref{fig:pgame} plots the average error rate for two configurations, $B=6,D=6$ and $B=2,D=16$, with the average in each setting obtained over 500 trees. The results here appear encouraging as well, with $\BRUSH$ overtaking the other two algorithms more quickly on the deeper trees.

\section{SUMMARY}

We have introduced $\BRUSH$, a simple Monte-Carlo algorithm for online planning in MDPs that  guarantees exponential-rate reduction of the performance measures of interest, namely the simple regret and the probability of erroneous action choice. This improves over previous algorithms such as $\UCT$, which guarantee only polynomial-rate reduction of these measures. The algorithm has been formalized for finite horizon MDPs, and it was analyzed as such. However, our empirical evaluation shows that it also performs well on goal-driven MDPs and two-person games.

A few questions remain for future work. In the setting of $\gamma$-discounted MDPs with infinite horizons, a straightforward way to employ $\BRUSH$ is to fix a horizon $H$, use the algorithm as is, and derive guarantees on the aforementioned measures of interest by simply accounting for the additive gap of $\gamma^{H}R_{\max}/(1-\gamma)$ between the state/action values under horizon $H$ and those under an infinite horizon. However, this is not necessarily the best way to plan online for infinite-horizon MDPs, and thus this setting requires further inspection. Second, it is not unlikely that the state-space independent factors $\paramone{h}$, and $\paramtwo{h}$ in the guarantees of $\BRUSH$ can be  improved by employing more sophisticated combinations of exploration and estimation samples. Another important point to consider is the speed of convergence to the optimal action, as opposed to the speed of convergence to ``good" actions. $\BRUSH$ is geared towards identifying the optimal action, although in many large MDPs, ``good" is often the best one can hope for. To identify the optimal solution, $\BRUSH$ devotes samples equally to all depths. However, focusing on nodes closer to the root node may improve the quality of the recommendation if the planning time is severely limited. Finally, the core tree sampling scheme employed by $\BRUSH$ differs from the more standard scheme employed in previous work. While this difference plays a critical role in establishing the formal guarantees of $\BRUSH$, it is still unclear whether 
that difference is {\em necessary} for establishing exponential-over-time reduction of the performance measures.

\subsection*{Acknowledgements}
This work is partially supported by and carried out at the Technion-Microsoft  Electronic Commerce Research Center, as well as partially supported by the Air Force Office of Scientific Research, USAF, under grant number FA8655-12-1-2096.

\bibliographystyle{theapa}
\bibliography{../../../phd}

\newpage
\appendix
\appendixpage
\section{Proof of Theorem~\ref{thm:brush_regret}}\label{app:proofs}

The proof of Theorem~\ref{thm:brush_regret} relies on the inductive assumption with respect to the correctness of Lemma~\ref{thm:brush_qaccuracy}, as well as on several auxiliary claims that we prove in what follows. The dependence diagram below depicts the overall flow of the proof, with the more central claims being depicted with rectangular nodes.

\begin{center}
\begin{minipage}{\textwidth}
\includegraphics[width=\textwidth]{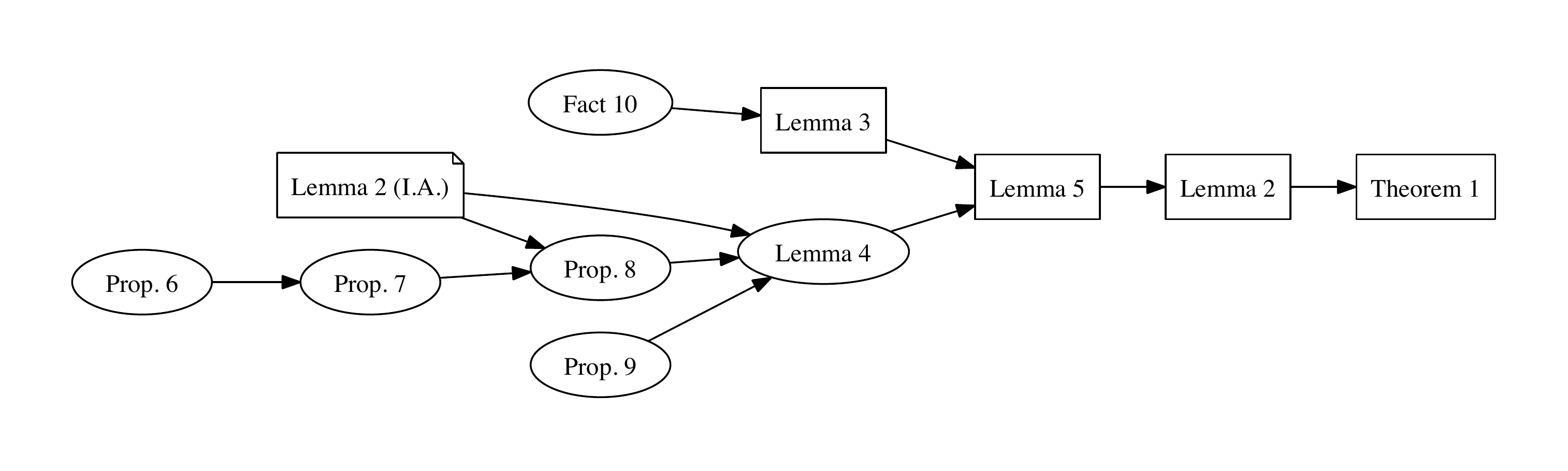}
\end{minipage}
\end{center}

\begin{proposition}[Concentration inequality for negative-binomial distributions]\label{lemma:nb_cdf}
Let $\negbinomial\left(t,p\right)$ be a random variable with negative-binomial distribution.
\begin{equation}
\mathbb{P}\left\{\negbinomial\left(t,p\right)\leq\frac{3t}{4p}\right\}\leq e^{-\frac{tp}{6}}
\end{equation}
\end{proposition}
\begin{proof}
It is well known that the event in which the number of Bernoulli trials required to obtain the $t$-th success is smaller than some positive integer $b$ is equivalent to the event that the number of successes in $b$ Bernoulli trials is at least $t$. Therefore, for any $0<\delta<1$,
\begin{equation}
\begin{split}
    \mathbb{P}\left\{\negbinomial\left(t,p\right)\leq \delta\frac{t}{p}\right\}&=\mathbb{P}\left\{\binomial\left(\delta\frac{t}{p},p\right)\geq t\right\}\\
    &=\mathbb{P}\left\{\binomial\left(\delta\frac{t}{p},p\right)\geq  t\delta + \left(t-t\delta\right)\right\}\\
    &\leq e^{-\frac{2t^2(1-\delta)^2p}{t\delta}}\\
    & \text{by the Hoeffding inequality},
\end{split}
\end{equation}
and choosing $\delta=\frac{3}{4}$ yields the result.
\end{proof}

\begin{proposition}[Number of Child Samples Bound]\label{lemma:child_samples_bound}
Let $\BRUSH$ be called on a state $\initstate$ of an MDP $\mdp{\states,\actions,\transp,\reward}$ with rewards in $\left[0,1\right]$ and finite horizon $H$. Let  $(s,h)$ be a node reachable from $(\initstate,H)$, and in turn, $(s',h')$ be a node reachable from $(s,h)$ via an action sequence that starts with applying action $a$ at $s$. Then, for any $a'\in A\left(s'\right)$, we have
\begin{equation}
    \mathbb{P}\left\{n_{h'}\left(s',a'\right)\leq \frac{p_{h'}\left(s',a'\right)}{4p_h\left(s,a\right)}t\;\middle|\;n_{h}\left(s,a\right)=t\right\}\leq2 e^{-t \frac{p_{h'}\left(s',a'\right)^2}{6p_h\left(s,a\right)}},
\end{equation}
where $p_h\left(s,a\right)$ is the probability that an $(H-h)$-iteration of $\BRUSH$ will issue a sample, whose exploration phase ends with applying action $a$ at state $s$ with $h$ steps still to go.
\end{proposition}

\begin{proof}
By the choice of the switching point function of $\BRUSH$ as in Eq.~\ref{eq:nh1}, the number of samples of action $a'$ in the descendant node $\left(s',h'\right)$ between two consecutive samples of action $a$ in node $\left(s,h\right)$ is distributed according to \begin{equation}
\label{eq:geomber}
\sum_{i=1}^{1+\gamma}\beta_i,
\end{equation}
 where $\gamma\sim\geomdist\left(p_h(s,a)\right)$ and $\beta_i\sim\bernulli\left(p_{h'}(s',a')\right)$ are all independent random variables. Indeed, for every pair of consecutive iterations $n<n'$ with $\spfunc(n)=\spfunc(n')=H-h$,
\begin{enumerate}[(i)]
\item there is exactly one iteration $n < n'' < n'$ with $\spfunc(n'')=H-h'$, and
\item the number of $(H-h)$-iterations between two consecutive $(H-h)$-iterations that finish their exploration phase with applying action $a$ at $s$ is geometric.
\end{enumerate}
Putting (i) and (ii) together, the number of $(H-h')$-iterations between a pair of consecutive $(H-h)$-iterations that  finish their exploration phase with applying action $a$ at $s$ is also geometric.
In turn, the probability that an $(H-h')$-iteration will
finish its exploration phase with applying $a'$ at $s'$ is  $p_{h'}\left(s',a'\right)$,
and thus the number of $(H-h')$-iterations that finish their exploration phase with applying $a'$ at $s'$ between a pair of consecutive $(H-h)$-iterations that  finish their exploration phase with applying action $a$ at $s$ is distributed as in Eq.~\ref{eq:geomber}.

Similarly, it can be shown that the (conditioned) random variable $$n_{h'}\left(s',a'\right) \mid n_h\left(s,a\right)=t$$ is distributed according to
\begin{equation}
\label{eq:geomber2}
\sum_{i=1}^{t+\gamma_t}\beta_i,
\end{equation}
where $\gamma_t\sim\negbinomial\left(t,p_h\left(s,a\right)\right)$, $\beta_i\sim\bernulli\left(p_{h'}\left(s',a'\right)\right)$, and all $\gamma_t$ and $\beta_i$ are independent.

Therefore, denoting $p_{h}(s,a)$ and $p_{h'}(s',a')$ by $p_{h}$ and $p_{h'}$, respectively, for short, we have
\begin{equation}
\label{eq:nsa_bound}
\begin{split}
    \mathbb{P}&\left\{n_{h'}\left(s,a\right)\leq\frac{t}{4p_h}p_{h'}\;\middle|\;n_{h'}\left(s,a\right)=t\right\}\\
    &\stackrel{\text{\tiny Eq.\ref{eq:geomber2}}}{=}\mathbb{P}\left\{\sum_{i=1}^{t+\gamma_t}\beta_i\leq \frac{t}{4p_h}p_{h'} \right\}\\
    &\leq\mathbb{P}\left\{\gamma_t\leq \frac{3t}{4p_h}\right\}
    +\sum_{x=\frac{3t}{4p_h}+1}^{\infty}\mathbb{P}\left\{\sum_{i=1}^{t+x}\beta_i\leq\frac{t}{4p_h}p_{h'}\right\}\mathbb{P}\left\{\gamma_t=x\right\}\\
    &=\mathbb{P}\left\{\gamma_t\leq \frac{3t}{4p_h}\right\}+ \sum_{x=\frac{3t}{4p_h}+1}^{\infty}\mathbb{P}\left\{\binomial\left(\left(t+x\right),p_{h'}\right)\leq\frac{t}{4p_h}p_{h'}\right\}\mathbb{P}\left\{\gamma_t=x\right\}\\
    &\text{\scriptsize since $\beta_i$ are all independent Bernoulli variables with common parameter $p_{h'}$}\\
    &=\mathbb{P}\left\{\gamma_t\leq \frac{3t}{4p_h}\right\}+ \sum_{x=\frac{3t}{4p_h}+1}^{\infty}\mathbb{P}\left\{\binomial\left(\left(t+x\right),p_{h'}\right)\leq\left(t+x\right)p_{h'}-\delta_x\right\}\mathbb{P}\left\{\gamma_t=x\right\},
\end{split}
\end{equation}
where $\delta_x=\frac{4xp_h-t(1-4p_h)}{4p_h}p_{h'}$.

Given that, for all $x\geq\frac{3t}{4p_h}$, we have
\begin{equation}
\label{eq:gamma_bound}
\begin{split}
    \mathbb{P}&\left\{\binomial\left(\left(t+x\right),p_{h'}\right)\leq\left(t+x\right)p_{h'}-\delta_x\right\}\leq  e^{-\frac{2\delta_x^2}{(t+x)}}\\
    &\hspace{0.5cm}\text{\scriptsize by the Hoeffding inequality, applicable here since $\delta_x=\frac{4xp_h-t(1-4p_h)}{4p_h}p_{h'}\geq
    \frac{t(2+4p_h)}{4p_h}p_{h'}\geq0$}\\
    &\leq e^{-t\left(2+\frac{1}{2p_h}\right)p_{h'}^2}\\
    &\hspace{0.5cm}\text{\scriptsize since $\frac{\delta_x^2}{t+x}\geq t+\frac{tp_{h'}^2}{4p_h}$}\\
    &\leq e^{-\frac{t}{2p_h}p_{h'}^2}.
\end{split}
\end{equation}
Plugging Eq.~\ref{eq:gamma_bound} into Eq.~\ref{eq:nsa_bound}, we obtain
\begin{equation}
\begin{split}
    \mathbb{P}&\left\{n_{g}\left(s,a\right)\leq\frac{t}{4p_h}p_{h'}\;\middle|\;n_{g}\left(s,a\right)=t\right\}\\
    & \leq \mathbb{P}\left\{\gamma_t\leq \frac{3t}{4p_h}\right\}+\sum_{x=\frac{3t}{4p_h}+1}^{\infty} e^{-\frac{p_{h'}^2t}{2p_h}}\mathbb{P}\left\{\gamma_t=x\right\}\\
    &\stackrel{\text{\tiny Prop.\ref{lemma:nb_cdf}}}{\leq} e^{-\frac{tp_h}{6}}+\sum_{x=\frac{3t}{4p_h}+1}^{\infty} e^{-\frac{p_{h'}^2t}{2p_h}}\mathbb{P}\left\{\gamma_t=x\right\}\\
    &\leq e^{-\frac{tp_h}{6}}+e^{-\frac{tp_{h'}^2}{2p_h}}\\
    &\leq 2e^{-\frac{tp_{h'}^2}{6p_h}}.
\end{split}
\end{equation}
\end{proof}


\begin{proposition}
\label{prop:ichs}
Let $\BRUSH$ be called on a state $\initstate$ of an MDP $\mdp{\states,\actions,\transp,\reward}$ with rewards in $\left[0,1\right]$ and finite horizon $H$. Let  $(s,h+1)$ be a node reachable from $(\initstate,H)$, and in turn, $(s',h')$ be a node reachable from $(s,h+1)$. If Lemma~\ref{thm:brush_qaccuracy} holds for horizon $h$, then, for any $a\in \actions{(s)}$, $a'\in \actions{(s')}$, and $t\geq 1$,
\begin{equation}
\label{eq:ichs2}
\mathbb{P}\left\{\estq{s',a'}{h'}-\optq{s',a'}{h'}\geq {d\over 2}\;\middle|\;n_{h+1}\left(s,a\right)=t\right\}\leq\left(2+\paramone{h'}\right)e^{-t\paramtwo{h'}\frac{p^{h+1-h'}}{6K^{h+1-h'}}},
\end{equation}
and
\begin{equation}
\label{eq:ichs4}
\mathbb{P}\left\{\estq{s',a'}{h'}-\optq{s',a'}{h'}\leq {- {d\over 2}}\;\middle|\;n_{h+1}\left(s,a\right)=t\right\}\leq\left(2+\paramone{h'}\right)e^{-t\paramtwo{h'}\frac{p^{h+1-h'}}{6K^{h+1-h'}}}.
\end{equation}
\end{proposition}

\begin{proof}
The proof for the two pairs of equations is identical, and thus we explicitly prove here only Eq.~\ref{eq:ichs2}.
In what follows, we use $p_{h}(s,a)$ and $p_{h'}(s',a')$ as defined in Proposition~\ref{lemma:child_samples_bound}, and here as well denote them by $p_{h}$ and $p_{h'}$, respectively, for short.  Similarly, by $\optqS{s',a'}{h'}$, $\estqS{s',a'}{h'}$, $\numsS{s',a'}{h'}$, and $\numsS{s,a}{h+1}$ we refer to $\optqS{s',a'}{h'}$, $\estq{s',a'}{h'}$, $\nums{s',a'}{h'}$, and $\nums{s,a}{h+1}$, respectively, for short.

\begin{equation}\label{eq:q_accuracy}
\begin{split}
    \mathbb{P}&\left\{\estqS{s',a'}{h'}-\optqS{s',a'}{h'}\geq {d\over 2}\;\middle|\;\numsS{s,a}{h+1}=t\right\}\\
    &\leq \mathbb{P}\left\{\numsS{s',a'}{h'}\leq \frac{tp_{h'}}{4p_{h+1}}\;\middle|\;\numsS{s,a}{h+1}=t\right\} \;+\;\mathbb{P}\left\{\estqS{s',a'}{h'}-\optqS{s',a'}{h'}\geq {d\over 2},\numsS{s',a'}{h'}> \frac{tp_{h'}}{4p_{h+1}}\;\middle|\;\numsS{s,a}{h+1}=t\right\}\\
    &\stackrel{\text{\tiny Prop.\ref{lemma:child_samples_bound}}}{\leq} 2e^{-\frac{p_{h'}^2t}{6p_{h+1}}} \;+
    \sum_{\tau=\frac{tp_{h'}}{4p_{h+1}}}^{\infty}\mathbb{P}\left\{\estqS{s',a'}{h'}-\optqS{s',a'}{h'}\geq {d\over 2}\;\middle|\;\numsS{s',a'}{h'}=\tau\right\} \mathbb{P}\left\{ \numsS{s',a'}{h'}=\tau\;\middle|\;\numsS{s,a}{h+1}=t \right\}\\
    &\stackrel{\text{\tiny I.A./Eq.\ref{eq:assmpt_eq}}}{\leq} 2e^{-\frac{tp_{h'}^2}{6p_{h+1}}} +
    \sum_{\tau=\frac{tp_{h'}}{4p_{h+1}}}^{\infty}\paramone{h'}e^{-\tau\paramtwo{h'}}\mathbb{P}\left\{ \numsS{s',a'}{h'}=\tau\;\middle|\;\numsS{s,a}{h+1}=t \right\}\\
    &\leq 2e^{-\frac{tp_{h'}^2}{6p_{h+1}}}+\paramone{h'}e^{-\paramtwo{h'}\frac{tp_{h'}}{4p_{h+1}}}.
\end{split}
\end{equation}
Consider the fraction $\frac{p_{h'}}{p_{h+1}}$, and recall that $(s',h')$ is a descendant of $(s,h)$. The latter implies that
\[
p_{h'} \geq p_{h+1}\left(\frac{p}{K}\right)^{h+1-h'},
\]
and thus $\frac{p_{h'}}{p_{h+1}} \geq \left(\frac{p}{K}\right)^{h+1-h'}$.
Continuing now Eq.~\ref{eq:q_accuracy},
by Eq.~\ref{eq:c_recursive_carmel2},
\[
c_{h'}'=\frac{3d^{2(h'-1)} p^{H+h'-1}}{32^{h'-1}((h')!)^2 K^{H+h'-1}} <
\left(\frac{p}{K}\right)^{H+h'-1} \leq  \left(\frac{p}{K}\right)^{H-h'} \leq p_{h'},
\]
and Eq.~\ref{eq:q_accuracy} under $\paramtwo{h'} < p_{h'}$ implies
\begin{equation}\label{eq:q_accuracy2}
\begin{split}
2e^{-\frac{tp_{h'}^2}{6p_{h+1}}}+\paramone{h'}e^{-\paramtwo{h'}\frac{tp_{h'}}{4p_{h+1}}}
&\leq \left(2+\paramone{h'}\right)e^{-t\paramtwo{h'}\frac{p_{h'}}{6p_{h+1}}}\\
&\leq \left(2+\paramone{h'}\right)e^{-t\paramtwo{h'}\frac{p^{h+1-h'}}{6K^{h+1-h'}}}.
\end{split}
\end{equation}

\end{proof}

\begin{proposition}[Expected accumulated rewards]\label{lemma:accum_rewards}
Let $\mdp{\states,\actions,\transp,\reward}$ be an MDP, and let $X$ be the accumulated reward of a sample $$\probe=\langle s, a,s_{1},a_{1},s_{h},a_{h},s_{h+1}\rangle,$$ started with taking action $a \in A$ in state $s \in \states$, and continued with additional $h$ steps, in which actions are chosen according to some arbitrary (possibly randomized) policy $\pi$. Let
\begin{itemize}
\item $E_{\pi,h+1}\left(s,a\right)$ denote the event in which, after $a$, $\probe$ is sampled along the optimal actions, that is, for $i\in\range{h}$, $a_{i} = \optpolicy_{h+1-i}(s_{i})$, and
\item $\delta_{\pi,h+1}\left(s,a\right)=Q_{h+1}\left(s,a\right)-\mathbb{E}\left[X\right]$.
\end{itemize}Then,
\begin{eqnarray}\label{eq:prob_bounds}
    \mathbb{P}\left\{\neg E_{\pi,h+1}\left(s,a\right)\right\}&\leq&\sum_{i=1}^h\mathbb{P}\left\{\pi_{h+1-i}\left(s_i\right)\neq\pi_{h+1-i}^*\left(s_i\right)\right\},\\
    \label{eq:prob_bounds1}
    \delta_{\pi,h+1}\left(s,a\right)&=& \sum_{i=1}^{h}\mathbb{E}\left[\regret{s_i,\pi_{h+1-i}(s_i)}{h+1-i}\right].
\end{eqnarray}
\end{proposition}

\begin{proof}
The proof of Eq.~\ref{eq:prob_bounds} is straightforward by the union bound.
To prove Eq.~\ref{eq:prob_bounds1}, we note that for any state/steps-to-go pair $(s,h) \in \states\times\range{H}$, 
we have
\[
\expectation_{\pi,s'}\left[R\left(s,\pi_h(s),s'\right)\right]=\expectation_{\pi}\left[\optq{s,\pi_h(s)}{h}\right]-\expectation_{\pi,s'}\left[\optq{s',\optpolicy(s',h-1)}{h-1}\right].
\]
Using that, we obtain a telescopic series that yields
\[
\begin{split}
\mathbb{E}\left[X\right]= &\; \mathbb{E}_{\pi, s_{1}:s_{h}}\left[R\left(s,a,s_1\right)+\sum_{i=1}^{h}R\left(s_i,\pi_{h+1-i}(s_i),s_{i+1}\right)\right]\\
= &\; \optq{s,a}{h+1}-
\expectation_{s_{1}}\left[\optq{s_1,\optpolicy(s_{1},h)}{h}\right]+\\
&\; \sum_{i=1}^{h}\left(\mathbb{E}_{\pi,s_{1}:s_{i}}\left[\optq{s_i,\pi_{h+1-i}(s_i)}{h+1-i}\right]-\mathbb{E}_{\pi,s_{1}:s_{i+1}}\left[\optq{s_{i+1},\optpolicy(s_{i+1},h-i)}{h-i}\right]\right)\\
=&\; \optq{s,a}{h+1}-\sum_{i=1}^{h}\mathbb{E}_{\pi,s_{1}:s_{i}}\left[\regret{s_i,\pi_{h-i+1}(s_{i})}{h+1-i}\right].
\end{split}
\]
\end{proof}


\subsection*{Proof of Lemma~\ref{thm:probe_bounds}:}
\begin{silentproof}
By Definition~\ref{def:Xevent}, the event $E_{t,h+1}\left(s,a\right)$ corresponds to a sample $$\probe=\langle s, a,s_{1},a_{1},s_{h},a_{h},s_{h+1}\rangle,$$
obtained by taking action $a$ at state $s$, reachable from $\initstate$ with $h+1$ steps still to go, and then following the policy $\reca{t}$, induced by running $\BRUSH$ on $\initstate$ until exactly $t>0$ samples finish their exploration phase with applying action $a$ at $s$ with $h$ steps still to go.
From Proposition~\ref{lemma:accum_rewards}, denoting $\fake{\imath}\stackrel{\triangle}{=} h+1-i$, we have
\begin{equation}
\label{eq:l3prof}
\begin{split}
\mathbb{P}\left\{\neg E_{t,h+1}\left(s,a\right)\right\}&\leq\sum_{i=1}^h\mathbb{P}\left\{\pi_{\fake{\imath}}\left(s_i\right)\neq\pi_{\fake{\imath}}^*\left(s_i\right)   \;\middle|\;n_{h+1}\left(s,a\right)=t\right\}\\
&\leq\sum_{i=1}^h\sum_{a'\neq\pi_{\fake{\imath}}^*\left(s_i\right)}\mathbb{P}\left\{\widehat{Q}_{\fake{\imath}}\left(s_i,a'\right)>\widehat{Q}_{\fake{\imath}}\left(s_i,\pi_{\fake{\imath}}^*\left(s_i\right)\right)\;\middle|\;n_{h+1}\left(s,a\right)=t \right\}\\
&\leq\sum_{i=1}^h\sum_{a'\neq\pi_{\fake{\imath}}^*\left(s_i\right)}\left[ \mathbb{P}\left\{\widehat{Q}_{\fake{\imath}}\left(s_i,a'\right)-Q_{\fake{\imath}}\left(s_i,a\right)\geq\frac{d}{2}\;\middle|\;n_{h+1}\left(s,a\right)=t\right\} + \right.\\
&\hspace{2.6cm}\left. \mathbb{P}\left\{\widehat{Q}_i\left(s_i,\pi_{\fake{\imath}}^*\left(s_i\right)\right)-Q_i\left(s_i,\pi_{\fake{\imath}}^*\left(s_i\right)\right)\leq -\frac{d}{2}\;\middle|\;n_{h+1}\left(s,a\right)=t\right\}\right]\\
&\stackrel{\text{\tiny Prop.\ref{prop:ichs}}}{\leq}
\sum_{i=1}^h 2K\left(2+\paramone{i}\right)e^{-t\paramtwo{i}\frac{p^{\fake{\imath}}}{6K^{\fake{\imath}}}}\\
&\stackrel{{\tiny (\ast)}}{\leq} 2Kh\left(2+\paramone{h}\right) e^{-t\paramtwo{h}\frac{p}{6K}}.
\end{split}
\end{equation}
The last inequality $(\ast)$ in Eq.~\ref{eq:l3prof} holds because,
by assuming Lemma~\ref{thm:brush_qaccuracy} for horizon $h$, for $i \in \range{h}$, it can be straightforwardly derived from Eqs.~\ref{eq:c_recursive_carmel1} and~\ref{eq:c_recursive_carmel2} that $c_{i}> c_{i-1}$ and $c_i'<\frac{p}{K}c_{i-1}'$.

Similarly,
\begin{equation}
\begin{split}
\delta_{t,h+1}\left(s,a\right)&\leq\sum_{i=1}^{h}\mathbb{E}\regret{s_i,\reca{t}(s_{i})}{\fake{\imath}}\\
&\leq\sum_{i=1}^h \left[ i \sum_{a'\neq\pi_{\fake{\imath}}^*\left(s_i\right)}\mathbb{P}\left\{\widehat{Q}_i\left(s_i,a'\right)>\widehat{Q}_i\left(s_i,\pi_{\fake{\imath}}^*\left(s_i\right)\right)\;\middle|\;n_{h+1}\left(s,a\right)=t\right\} \right]\\
&\leq\sum_{i=1}^h \left[ h \sum_{a'\neq\pi_{\fake{\imath}}^*\left(s_i\right)}\mathbb{P}\left\{\widehat{Q}_i\left(s_i,a'\right)>\widehat{Q}_i\left(s_i,\pi_{\fake{\imath}}^*\left(s_i\right)\right)\;\middle|\;n_{h+1}\left(s,a\right)=t\right\} \right]\\
&\stackrel{\text{\tiny Prop.\ref{prop:ichs}}}{\leq}
\sum_{i=1}^h 2Kh\left(2+\paramone{i}\right)e^{-t\paramtwo{i}\frac{p^{\fake{\imath}}}{6K^{\fake{\imath}}}}\\
&\leq 2Kh^{2}\left(2+\paramone{h}\right) e^{-t\paramtwo{h}\frac{p}{6K}}.
\end{split}
\end{equation}
\end{silentproof}

\begin{fact}\label{claim:exponent_bound}
Let Z be a random variable with support $[a,b]$ and $\mathbb{E}\left[Z\right]=0$. Then, for any $\lambda\in\mathbb{R^{+}}$, $$\mathbb{E}\left[\exp(\lambda Z)\right]\leq\exp(\frac{(b-a)^2\lambda^2}{8}).$$
\end{fact}
This result is well known due to Hoeffding.


\subsection*{Proof of Lemma~\ref{lemma:azuma_mod} (Modified Hoeffding-Azuma inequality):}

\begin{silentproof}
Let $E_{t}$ be the event that $\mathbb{E}\left[X_t\;\middle|\;X_1,\ldots,X_{t-1}\right] = \mu$, and let $$Y_{t} \triangleq X_t-\mu\;|\;X_1\left(\omega\right),\ldots,X_{t-1}\left(\omega\right).$$ The random variable $Y_{t}$ is bounded by $h-\mu\geq Y_{t}\geq-\mu$, and furthermore,  for $\omega\in E_t$, $\mathbb{E}Y_{t}=0$. Therefore, using Fact~\ref{claim:exponent_bound}, for all $\omega\in E_t$ and $\lambda\in\mathbb{R}^{+}$, it holds that
\begin{equation}
\mathbb{E}\left[e^{\lambda Y}\right]\leq e^{\frac{\lambda^2h^2}{8}}.
\end{equation}
Moreover,
\begin{equation}\label{eq:exp_exp}
\begin{split}
\mathbb{E}&\left[e^{\lambda \sum_{i=1}^t\left(\mu-X_i\right)}\right]\\
&=\mathbb{E}_{E_t}\left[e^{\lambda \sum_{i=1}^t\left(\mu-X_i\right)}\right] + \mathbb{E}_{\neg E_t}\left[e^{\lambda \sum_{i=1}^t\left(\mu-X_i\right)}\right]\\
&=\mathbb{E}_{E_t}\left[ \mathbb{E}\left[ e^{\lambda \sum_{i=1}^t\left(\mu-X_i\right)} \;\middle|\; X_{1},\ldots,X_{t-1} \right]  \right] +
\mathbb{E}_{\neg E_t}\left[e^{\lambda \sum_{i=1}^t\left(\mu-X_i\right)}\right]
\\
&=\mathbb{E}\left[ e^{\lambda \sum_{i=1}^{t-1}\left(\mu-X_i\right)} \cdot \mathbb{E}\left[ e^{\lambda (\mu-X_t)} \;\middle|\; X_{1},\ldots,X_{t-1} \right]  \right] +
\mathbb{E}_{\neg E_t}\left[e^{\lambda \sum_{i=1}^t\left(\mu-X_i\right)}\right]
\\
&\leq e^{\frac{\lambda^2h^2}{8}}\expectation \left[e^{\lambda \sum_{i=1}^{t-1}(\mu-X_i)}\right] + \simpleprob \left\{\neg E_t\right\}e^{\lambda th}\\
&\stackrel{\text{Eq.\ref{e:azuma1}}}{\leq} e^{\frac{\lambda^2h^2}{8}}\expectation \left[e^{\lambda \sum_{i=1}^{t-1}(\mu-X_i)}\right] +  c_p e^{t\left(\lambda h-c_e\right)}\\
&\stackrel{\text{($\ast$)}}{\leq} e^{\frac{\lambda^2h^2}{8}t}   +
\sum_{\tau=1}^t e^{\frac{\lambda^2h^2}{8}\left(t-\tau\right)} \cdot c_p e^{\tau\left(\lambda h-c_e\right)}\\
&\text{\footnotesize by the auxiliary step in Eqs.~\ref{eq:recursion}-\ref{eq:recursion3} below }\\
&\leq e^{\frac{\lambda^2h^2t}{8}}\left[ 1+c_p \sum_{\tau=1}^{\infty} e^{-\frac{\lambda^2h^2\tau}{8}} e^{\tau\left(\lambda h-c_e\right)}\right].
\end{split}
\end{equation}
Considering the recursion
\begin{equation}
\label{eq:recursion}
	f\left(t\right)=\theta f\left(t-1\right)+g\left(t\right),
\end{equation}
it is easy to verify that, for all $0 \leq c < t$, 
\begin{equation}
\label{eq:recursion2}
	f\left(t\right)=\theta^{t-c}f\left(c\right)+\sum_{\tau=c+1}^{t}\theta^{t-\tau}g\left(\tau\right).
\end{equation}
Given that, the bound ($\ast$) in Eq.~\ref{eq:exp_exp} is obtained by setting
\begin{equation}
\label{eq:recursion3}
\begin{split}
\theta &=e^{\frac{\lambda^2h^2}{8}},\\
f\left(t\right) &=\expectation \left[e^{\lambda \sum_{i=1}^t (\mu-X_i)}\right],\\
g\left(t\right) &=c_pe^{t\left(\lambda h-c_e\right)}.\\
\end{split}
\end{equation}

Now, by Markov inequality,  for any $\lambda>0$,
\begin{equation}
\begin{split}
\mathbb{P}&\left\{\mu t-\sum_{i=1}^tX_i\geq t\delta\right\}\\
&\leq e^{-\lambda t\delta}\mathbb{E}\left[e^{\lambda \sum_{i=1}^t\left(\mu-X_i\right)}\right]\\
&\stackrel{\text{Eq.\ref{eq:exp_exp}}}{\leq} e^{-\lambda t\delta} e^{\frac{\lambda^2h^2t}{8}} \left[ 1+ c_p\sum_{\tau=1}^{\infty} e^{-\frac{\lambda^2h^2\tau}{8}}e^{\tau\left(\lambda h-c_e\right)}\right]\\
&= e^{-\frac{2\delta^2 c_e}{h^2} t} e^{\frac{\delta^2 c_e^2}{2h^2}t} \left[ 1+ c_p\sum_{\tau=1}^{\infty} e^{-\frac{\delta^2 c_e^2}{2h^2}\tau} e^{\tau\left(\frac{2\delta c_e}{h}-c_e\right)}\right]\\
&\text{\footnotesize by setting $\lambda=\frac{2\delta c_e}{h^2}$}\\
&\leq e^{-\frac{3\delta^2 c_e}{2h^2}t }\left[ 1+ c_p\frac{2h^2}{\delta^2 c_e^2} e^{-\frac{\delta^2 c_e^2}{2h^2}}\right]\\
&\text{\footnotesize since $\delta<\frac{h}{2}$}\\
&\leq e^{-\frac{3\delta^2 c_e}{2h^2}t} \left[ 1+c_p\frac{2h^2}{\delta^2 c_e^2} \right].
\end{split}
\end{equation}
The second bound can be proven in much the same way.
\end{silentproof}
\begin{disc}\label{dsc:estimate_bound}
Note that the above bound was obtained for a particular choice of $\lambda$ that maximizes the coefficient term in the exponent. Other choices of $\lambda$ may result in a smaller coefficient in the exponent, but also a decreased leading constant. In particular, setting $\lambda=\frac{2\delta c_e\beta}{h^2}$, for any $0<\beta<1$, yields the following bound
\begin{equation}
\begin{split}
\mathbb{P}&\left\{\mu t-\sum_{i=1}^tX_i\geq t\delta\right\}\\
&\leq e^{-\frac{2\delta^2 c_e\beta}{h^2} t} e^{\frac{\delta^2 c_e^2\beta^2}{2h^2}t} \left[ 1+ c_p\sum_{\tau=1}^{\infty} e^{-\frac{\delta^2 c_e^2\beta^2}{2h^2}\tau} e^{\tau\left(\frac{2\delta c_e\beta}{h}-c_e\right)}\right]\\
&\leq e^{-\frac{3\delta^2 c_e\beta}{2h^2}t }\left[ 1+ \frac{c_p}{c_e\left(1-\beta\right)} e^{-\frac{c_e\left(1-\beta\right)}{2h^2}}\right].
\end{split}
\end{equation}
\end{disc}
\subsection*{Proof of Lemma~\ref{thm:induction_step}:}
\begin{silentproof}
Lemma~\ref{thm:probe_bounds} implies that, with probability approaching $1$ exponentially fast, the state-space samples issued at a level with $h+1$ steps-to-go are optimal. That is, their expectation equals the actual $Q$-value. Therefore, by Lemma~\ref{thm:probe_bounds}, we have
\begin{equation}
\mathbb{P}\left\{\mathbb{E}\left[X_{t,h+1}(s,a)\;\middle|\;X_{1,h+1}(s,a),\ldots,X_{t-1,h+1}(s,a)\right]\neq Q\left(s,a\right)\right\}\leq c_p e^{-c_et},
\end{equation}
where $c_p=2Kh(2+c_h)$ and $c_e=\frac{pc_{h}'}{6K}$. It is also easy to see that $0\leq X_i\leq h+1$, and thus the conditions of Lemma~\ref{lemma:azuma_mod} are satisfied. In turn, from Lemma~\ref{lemma:azuma_mod} for $\delta=\frac{d}{2}$ and random variables with support $[0,h+1]$,
\begin{equation}
\begin{split}
\mathbb{P}&\left\{\widehat{Q}_{h+1}\left(s,a\right)-Q_{h+1}\left(s,a\right)\geq \frac{d}{2}\;\middle|\;n_{h+1}\left(s,a\right)=t\right\}\\
&\leq \left[ 1 + c_p\frac{2(h+1)^2}{\left(\frac{d}{2}\right)^2 c_e^2} \right] \cdot e^{-\frac{3\left(\frac{d}{2}\right)^2 c_e}{2(h+1)^2}t}\\
&\leq e^{-\frac{d^2 pc_{h}'}{16(h+1)^2K}t}\left[ 1152 \cdot \frac{K^3(h+1)^3(2+c_h)}{d^2 p^2c_h'^2} \right],
\end{split}
\end{equation}
and, similarly,
\begin{equation}
\label{e:finishlemma4}
\begin{split}
\mathbb{P}&\left\{\widehat{Q}_{h+1}\left(s,a\right)-Q_{h+1}\left(s,a\right)\leq -\frac{d}{2}\;\middle|\;n_{h+1}\left(s,a\right)=t\right\}\\
&\leq e^{-\frac{d^2 pc_{h}'}{16(h+1)^2K}t}\left[ 1152 \cdot \frac{K^3(h+1)^3(2+c_h)}{d^2 p^2c_h'^2} \right]\\
&\leq e^{-\frac{d^2 pc_{h}'}{16(h+1)^2K}t}\left[ 3456 \cdot \frac{K^3(h+1)^3 c_h}{d^2 p^2c_h'^2} \right]\\
&\hspace{0.5cm} \text{\footnotesize since $2+c_{h} \leq 3c_{h}$.}
\end{split}
\end{equation}
\end{silentproof}

\subsection*{Proof of Lemma~\ref{thm:brush_qaccuracy}, induction step:}
\begin{silentproof}
Note that the proof of Lemma~\ref{thm:induction_step} is basically the proof of the induction step for the key part of Lemma~\ref{thm:brush_qaccuracy}, that is, Eq.~\ref{eq:assmpt_eq}. The only thing that remains to be finalized is the correctness of Eqs.~\ref{eq:c_recursive_carmel1} and~\ref{eq:c_recursive_carmel2} for $h+1$, and these can be verified by substitution of $c_{h}$ and $c_{h}'$ in Eq.~\ref{e:finishlemma4} by the respective expressions (for $h$) from Eqs.~\ref{eq:c_recursive_carmel1} and~\ref{eq:c_recursive_carmel2}.
\end{silentproof}

\subsection*{Proof of Theorem~\ref{thm:brush_regret}:}
\begin{silentproof}
The proof for our main results follows by using the same techniques as above.
Note that, by the Hoeffding inequality, after $n > 0$ iterations of $\BRUSH$, for each action $a\in A(\initstate)$, it holds that
\begin{equation}
\mathbb{P}\left\{n_H\left(\initstate,a\right)\leq \frac{n}{2KH}\right\}\leq e^{-\frac{1}{2K^{2}H}n}.
\end{equation}
Given that,
\begin{equation}
\begin{split}
\simpleprob &\left\{\reca{n}(\initstate,H)\neq\optpolicy(\initstate,H)\right\}\\
&\leq \sum_{a \neq \optpolicy(\initstate,H)}\left[
\mathbb{P}\left\{\widehat{Q}_H\left(\initstate,a\right)\geq Q_H\left(\initstate,a\right)+\frac{d}{2}\right\} + \right.\\
&\hspace{2.3cm}\left.\mathbb{P}\left\{\widehat{Q}_H\left(\initstate,\optpolicy(\initstate,H)\right)\leq Q_H\left(\initstate,\optpolicy(\initstate,H)\right)-\frac{d}{2}\right\}
\right].
\end{split}
\end{equation}
For a sub-optimal action $a$,
\begin{equation}
\begin{split}
\mathbb{P}&\left\{\widehat{Q}_H\left(\initstate,a\right)\geq Q_H\left(\initstate,a\right)+\frac{d}{2}\right\}\\
&=\sum_{t=1}^{\frac{n}{H}}\mathbb{P}\left\{\widehat{Q}_H\left(\initstate,a\right)\geq Q_H\left(\initstate,a\right)+\frac{d}{2}\;\middle|\;n_H\left(\initstate,a\right)=t\right\}\mathbb{P}\left\{n_H\left(\initstate,a\right)=t\right\}\\
&\leq\mathbb{P}\left\{n_H\left(\initstate,a\right)\leq \frac{n}{2KH}\right\}\\
&\hspace{1cm}+\sum_{t=1+\frac{n}{2KH}}^{\frac{n}{H}}\mathbb{P}\left\{\widehat{Q}_H\left(\initstate,a\right)\geq Q_H\left(\initstate,a\right)+\frac{d}{2}\;\middle|\;n_H\left(\initstate,a\right)=t\right\}\mathbb{P}\left\{n_H\left(\initstate,a\right)=t\right\}\\
&\stackrel{\text{Lemma~\ref{thm:brush_qaccuracy}}}{\leq} \;\;\; e^{-\frac{1}{2K^{2}H}n}\;+\sum_{t=1+\frac{n}{2KH}}^{\frac{n}{H}}  c_H e^{-c_H't} \cdot\mathbb{P}\left\{n_H\left(\initstate,a\right)=t\right\}\\
&\leq e^{-\frac{1}{2K^{2}H}n}+ c_H e^{-\frac{c_H'}{2KH}n}\\
&\leq 2c_H e^{-\frac{c_H'}{2KH}n}.
\end{split}
\end{equation}
Using exactly the same line of bounding, we obtain
\begin{equation}
\mathbb{P}\left\{\widehat{Q}_H\left(\initstate,\optpolicy(\initstate,H)\right)\leq Q_H\left(\initstate,\optpolicy(\initstate,H)\right)-\frac{d}{2}\right\} \leq 2c_H e^{-\frac{c_H'}{2KH}n},
\end{equation}
and thus
\begin{equation}
\label{eq:noname}
\simpleprob \left\{\reca{n}(\initstate,H)\neq\optpolicy(\initstate,H)\right\} \leq
4K c_H e^{-\frac{c_H'}{2KH}n}.
\end{equation}
Eqs.~\ref{eq:brush_gen_prob},~\ref{eq:cgenone}, and~\ref{eq:cgentwo} of Theorem~\ref{thm:brush_regret} are then obtained by substitution of $c_{H}$ and $c_{H}'$ in Eq.~\ref{eq:noname} with the respective expressions from Eqs.~\ref{eq:c_recursive_carmel1} and~\ref{eq:c_recursive_carmel2}. In turn, Eq.~\ref{eq:brush_gen_reg} of Theorem~\ref{thm:brush_regret} stems from Eqs.~\ref{eq:brush_gen_prob}, horizon $H$, and per-step rewards being in $[0,1]$.
\end{silentproof}

\newpage
\section{Proof of Theorem~\ref{thm:brue_alpha}}
\label{app:proofsalpha}
We first prove the modified Hoeffding-Azuma inequality for partial sums.
\subsection*{Proof of Lemma~\ref{lemma:azuma_mod_alpha} (Modified Hoeffding-Azuma inequality for partial sums):}
\begin{silentproof}
Let $E_{t}$ be the event that $\mathbb{E}\left[X_t\;\middle|\;X_1,\ldots,X_{t-1}\right] = \mu$, and let $$Y_{t} \triangleq X_t-\mu\;|\;X_1\left(\omega\right),\ldots,X_{t-1}\left(\omega\right).$$ The random variable $Y_{t}$ is bounded by $h-\mu\geq Y_{t}\geq-\mu$, and furthermore,  for $\omega\in E_t$, $\mathbb{E}Y_{t}=0$. Therefore, using Fact~\ref{claim:exponent_bound}, for all $\omega\in E_t$ and $\lambda\in\mathbb{R}^{+}$, it holds that
\begin{equation}
\mathbb{E}\left[e^{\lambda Y}\right]\leq e^{\frac{\lambda^2h^2}{8}}.
\end{equation}
Moreover,
\begin{equation}\label{eq:exp_exp_alpha}
\begin{split}
\mathbb{E}&\left[e^{\lambda \sum_{i=t-\lceil\alpha t\rceil}^t\left(\mu-X_i\right)}\right]\\
&=\mathbb{E}_{E_t}\left[e^{\lambda \sum_{i=t-\lceil\alpha t\rceil}^t\left(\mu-X_i\right)}\right] + \mathbb{E}_{\neg E_t}\left[e^{\lambda \sum_{i=t-\lceil\alpha t\rceil}^t\left(\mu-X_i\right)}\right]\\
&=\mathbb{E}_{E_t}\left[ \mathbb{E}\left[ e^{\lambda \sum_{i=t-\lceil\alpha t\rceil}^t\left(\mu-X_i\right)} \;\middle|\; X_{1},\ldots,X_{t-1} \right]  \right] +
\mathbb{E}_{\neg E_t}\left[e^{\lambda \sum_{i=t-\lceil\alpha t\rceil}^t\left(\mu-X_i\right)}\right]
\\
&=\mathbb{E}\left[ e^{\lambda \sum_{i=t-\lceil\alpha t\rceil}^{t-1}\left(\mu-X_i\right)} \cdot \mathbb{E}\left[ e^{\lambda (\mu-X_t)} \;\middle|\; X_{1},\ldots,X_{t-1} \right]  \right] +
\mathbb{E}_{\neg E_t}\left[e^{\lambda \sum_{i=t-\lceil\alpha t\rceil}^t\left(\mu-X_i\right)}\right]
\\
&\leq e^{\frac{\lambda^2h^2}{8}}\expectation \left[e^{\lambda \sum_{i=t-\lceil\alpha t\rceil}^{t-1}(\mu-X_i)}\right] + \simpleprob \left\{\neg E_t\right\}e^{\lambda h\lceil\alpha t\rceil}\\
&\stackrel{\text{Eq.\ref{e:azuma1}}}{\leq} e^{\frac{\lambda^2h^2}{8}}\expectation \left[e^{\lambda \sum_{i=t-\lceil\alpha t\rceil}^{t-1}(\mu-X_i)}\right] +  c_p e^{\lambda h\lceil\alpha t\rceil-c_e t}\\
&\stackrel{\text{($\ast$)}}{\leq} e^{\frac{\lambda^2h^2}{8}\lceil\alpha t\rceil}   +
\sum_{\tau=t-\lceil\alpha t\rceil}^t e^{\frac{\lambda^2h^2}{8}\left(t-\tau\right)} \cdot c_p e^{\lambda h\lceil\alpha\tau\rceil-c_e \tau}\\
&\text{\footnotesize by the auxiliary step in Eqs.~\ref{eq:recursion}-\ref{eq:recursion3new} below. }
\end{split}
\end{equation}
Considering the recursion
\begin{equation}
\label{eq:recursion_alpha}
	f\left(t\right)=\theta f\left(t-1\right)+g\left(t\right),
\end{equation}
it is easy to verify that, for all $0 \leq c < t$, 
\begin{equation}
\label{eq:recursion2_alpha}
	f\left(t\right)=\theta^{t-c}f\left(c\right)+\sum_{\tau=c+1}^{t}\theta^{t-\tau}g\left(\tau\right).
\end{equation}
Given that, the bound ($\ast$) in Eq.~\ref{eq:exp_exp} is obtained by setting
\begin{equation}
\label{eq:recursion3new}
\begin{split}
\theta &=e^{\frac{\lambda^2h^2}{8}},\\
f\left(t\right) &=\expectation \left[e^{\lambda \sum_{i=1}^t (\mu-X_i)}\right],\\
g\left(t\right) &=c_pe^{t\left(\lambda h-c_e\right)}.\\
\end{split}
\end{equation}

Now, by Markov inequality,  for any $\lambda>0$,
\begin{equation}
\begin{split}
\mathbb{P}&\left\{\mu \lceil\alpha t\rceil-\sum_{i=t-\lceil\alpha t\rceil}^tX_i\geq \lceil\alpha t\rceil\delta\right\}\\
&\leq e^{-\lambda\delta\lceil\alpha t\rceil}\mathbb{E}\left[e^{\lambda \sum_{i=t-\lceil\alpha t\rceil}^t\left(\mu-X_i\right)}\right]\\
&\stackrel{\text{Eq.\ref{eq:exp_exp}}}{\leq} e^{-\lambda\delta\lceil\alpha t\rceil} \left[e^{\frac{\lambda^2h^2}{8}\lceil\alpha t\rceil}   +
\sum_{\tau=t-\lceil\alpha t\rceil}^t e^{\frac{\lambda^2h^2}{8}\left(t-\tau\right)} \cdot c_p e^{\lambda h\lceil\alpha\tau\rceil-c_e \tau}\right]\\
&= e^{-\frac{2\delta^2 c_e}{h^2}\lceil\alpha t\rceil} \left[e^{\frac{\delta^2 c_e^2}{2h^2}\lceil\alpha t\rceil}   +
\sum_{\tau=t-\lceil\alpha t\rceil}^t e^{\frac{\delta^2 c_e^2}{2h^2}\left(t-\tau\right)} \cdot c_p e^{\frac{2\delta c_e}{h} \lceil\alpha\tau\rceil-c_e \tau}\right]\\
&\text{\footnotesize by setting $\lambda=\frac{2\delta c_e}{h^2}$}\\
&\leq e^{-\frac{2\delta^2 c_e}{h^2}\lceil\alpha t\rceil} \left[e^{\frac{\delta^2 c_e^2}{2h^2}\lceil\alpha t\rceil}   + e^{\frac{\delta^2 c_e^2}{2h^2}\lceil\alpha t\rceil}
\sum_{\tau=t-\lceil\alpha t\rceil}^t c_p e^{\frac{2\delta c_e}{h} \lceil\alpha\tau\rceil-c_e \tau}\right]\\
&\leq e^{-\frac{2\delta^2 c_e}{h^2}\lceil\alpha t\rceil} e^{\frac{\delta^2 c_e^2}{2h^2}\lceil\alpha t\rceil}\left[1 +
c_p\sum_{\tau=t-\lceil\alpha t\rceil}^t  e^{c_e \lceil\alpha\tau\rceil-c_e \tau}\right]\\
&\text{\footnotesize since $\delta<\frac{h}{2}$}\\
&\leq e^{-\frac{2\delta^2 c_e}{h^2}\lceil\alpha t\rceil} e^{\frac{\delta^2 c_e^2}{2h^2}\lceil\alpha t\rceil}\left[1 +
c_p \sum_{\tau=t-\lceil\alpha t\rceil}^t  e^{-c_e(1-\alpha)\tau}\right]\\
&\leq e^{-\frac{2\delta^2 c_e}{h^2}\lceil\alpha t\rceil} e^{\frac{\delta^2 c_e^2}{2h^2}\lceil\alpha t\rceil}\left[1 +
\frac{c_p}{c_e(1-\alpha)} e^{-c_e(1-\alpha)^2t}\right]\\
&\leq e^{-\frac{3\delta^2 c_e}{2h^2}\alpha t} \left[1 +
\frac{c_p }{c_e(1-\alpha)} e^{-c_e(1-\alpha)^2t}\right].
\end{split}
\end{equation}
\end{silentproof}

Obviously, at leaf nodes there is no point in choosing $\alpha<1$ since there is no bias. Therefore, for $h=1$ we can use the same constants $c_1=1$ and $c_1'=\frac{p^H}{K^H}$. Since $c_h'$ is decreasing with $h$, we have $c_h'\leq \frac{p^{H-h}}{K^{H-h}}$ for all $1\leq h\leq H$, and thus Lemma~\ref{thm:probe_bounds} is valid. Lemma~\ref{thm:induction_step} relies on the modified Hoeffding-Azuma inequality, which is no longer valid in the context of $\BRUSH\left(\alpha\right)$. Instead, we apply its modification, Lemma~\ref{lemma:azuma_mod_alpha} for partial sums, to prove the induction step
\begin{equation}
\begin{split}
\mathbb{P}&\left\{\widehat{Q}_{h+1}\left(s,a\right)-Q_{h+1}\left(s,a\right)\geq \frac{d}{2}\;\middle|\;n_{h+1}\left(s,a\right)=t\right\}\\
&\leq e^{-\frac{3d^2 pc_{h}'}{48Kh^2}\alpha t} \left[1 +
\frac{12K^2h(2+c_h) }{pc_{h}'(1-\alpha)} e^{-\frac{pc_{h}'(1-\alpha)^2}{6K}t}\right]
\end{split}
\end{equation}
and, similarly,
\begin{equation}
\label{e:finishlemma4new}
\begin{split}
\mathbb{P}&\left\{\widehat{Q}_{h+1}\left(s,a\right)-Q_{h+1}\left(s,a\right)\leq -\frac{d}{2}\;\middle|\;n_{h+1}\left(s,a\right)=t\right\}\\
&\leq e^{-\frac{3d^2 pc_{h}'}{48Kh^2}\alpha t} \left[1 +
\frac{12K^2h(2+c_h) }{pc_{h}'(1-\alpha)} e^{-\frac{pc_{h}'(1-\alpha)^2}{6K}t}\right].
\end{split}
\end{equation}
The induction step is satisfied, e.g., with $c_{h+1}'=\min\left\{\frac{3d^2 pc_{h}'\alpha}{48Kh^2},\frac{pc_{h}'(1-\alpha)^2}{6K}\right\}$ and $c_{h+1}=1 +
\frac{12K^2h(2+c_h) }{pc_{h}'(1-\alpha)}$.

Since $c_h$ is increasing in $h$ and $c_h'$ is decreasing in $h$, the term $\frac{12K^2h(2+c_h) }{pc_{h}'(1-\alpha)}$ also increases in $h$. The larger the constant grows, the more beneficial it might be to increase the exponent coefficient that multiplies that constant by decreasing $\alpha$ at the expense of decreasing the exponent coefficient that multiplies $1$. 
Clearly, the tradeoff depends also on $t$, the number of samples of action $a$ in node $\left(s,h\right)$. Therefore, as $h$ increases, smaller values of $\alpha$ would be more appealing.

\end{document}